\documentclass[10pt]{article}

\usepackage[utf8]{inputenc} 
\usepackage[T1]{fontenc}    

\usepackage{url}            
\usepackage{booktabs}       
\usepackage{amsfonts}       
\usepackage{nicefrac}       
\usepackage{microtype}      
\usepackage{fullpage}
\usepackage{mathtools}
\usepackage{xcolor}

\usepackage{microtype}
\usepackage{graphicx}
\usepackage{subfigure}
\usepackage{booktabs} 
\usepackage[paper=a4paper]{geometry}
\usepackage{hyperref}

\usepackage{amsmath}
\usepackage{amssymb}
\usepackage{mathtools}
\usepackage{amsthm}
\usepackage{enumitem}
 
\usepackage{natbib}
\usepackage{subfigure}
\usepackage{adjustbox}

\usepackage{commath}
\usepackage{wrapfig}

\usepackage{algorithm}
\usepackage{algorithmic}

\usepackage{amsfonts,bm,epsfig,graphicx,epstopdf,bbm}
\usepackage{setspace,latexsym,epsf,dsfont,caption,mathtools,color}
\usepackage{tabularx}
\usepackage{multirow}

\usepackage[capitalize,noabbrev]{cleveref}

\usepackage{tikz}

\theoremstyle{plain}
\newtheorem{theorem}{Theorem}[section]
\newtheorem{proposition}[theorem]{Proposition}
\newtheorem{lemma}[theorem]{Lemma}
\newtheorem{corollary}[theorem]{Corollary}
\theoremstyle{definition}
\newtheorem{definition}[theorem]{Definition}
\newtheorem{assumption}[theorem]{Assumption}
\theoremstyle{remark}
\newtheorem{remark}[theorem]{Remark}

\theoremstyle{claim}

\crefname{claim}{claim}{claims}
\crefname{assumption}{assumption}{assumptions}
 
\usepackage[textsize=tiny]{todonotes}

\usepackage{multirow}
\usepackage{pifont}

\newcommand{\Eb}{\mathbb{E}}
\newcommand{\Pb}{\mathbb{P}}
\newcommand{\Qb}{\mathbb{Q}}

\newcommand{\Dc}{\mathcal{D}}

\newcommand{\Tc}{\mathcal{T}}

\newcommand{\Ec}{\mathcal{E}}

\newcommand{\Pc}{\mathcal{P}}

\providecommand{\keywords}[1]
{
  \small	
  \textbf{\textit{Keywords---}} #1
}

\newcommand{\Db}{\mathbb{D}}

\newcommand{\mbf}{\mathbf{m}}
\newcommand{\Mbf}{\mathbf{M}}

\newcommand{\btau}{\boldsymbol{\tau}}

\newcommand{\KL}{\mathrm{kl}}
\newcommand{\TV}{\mathtt{TV}}

\allowdisplaybreaks

\title{Robust Offline Reinforcement Learning for Non-Markovian Decision Processes}

\author{
 Ruiquan Huang \\
 {\small Department of Electrical Engineering}\\
 \small The Pennsylvania State University\\
 \small University Park, PA 16801\\
 \small \texttt{rzh5514@psu.edu} 
 \and
 Yingbin Liang \\
 \small Department of Electrical and Computer Engineering\\
 \small The Ohio State University\\
 \small Columbus, OH 43210\\
 \small \texttt{liang.889@osu.edu}
 \and 
 Jing Yang \\
 \small Department of Electrical Engineering\\
 \small The Pennsylvania State University\\
 \small University Park, PA 16801\\
 \small \texttt{yangjing@psu.edu}
}
\date{}

\begin{document}

\maketitle

\begin{abstract}
Distributionally robust offline reinforcement learning (RL) aims to find a policy that performs the best under the worst environment within an uncertainty set using an offline dataset collected from a nominal model. While recent advances in robust RL focus on Markov decision processes (MDPs), robust non-Markovian RL is limited to planning problem where the transitions in the uncertainty set are known. In this paper, we study the learning problem of robust offline non-Markovian RL. Specifically, when the nominal model admits a low-rank structure, we propose a new algorithm, featuring a novel dataset distillation and a lower confidence bound (LCB) design for robust values under different types of the uncertainty set. We also derive new dual forms for these robust values in non-Markovian RL, making our algorithm more amenable to practical implementation. By further introducing a novel type-I concentrability coefficient tailored for offline low-rank non-Markovian decision processes, we prove that our algorithm can find an $\epsilon$-optimal robust policy using $O(1/\epsilon^2)$ offline samples. Moreover, we extend our algorithm to the case when the nominal model does not have specific structure. With a new type-II concentrability coefficient, the extended algorithm also enjoys polynomial sample efficiency under all different types of the uncertainty set.

\end{abstract}

\keywords{Robust RL, Offline RL, POMDP, Non-markovian Decision Processes}

\section{Introduction}

In reinforcement learning (RL)~\citep{sutton2018reinforcement}, an agent learns to make decisions by interacting with an unknown environment and maximizes the total received reward. This learning framework has found wide applications such as video games~\citep{ye2020mastering}, robotics~\citep{ibarz2021train}, and language model fine-tuning~\citep{ouyang2022training}. However, in many real-world scenarios such as autonomous driving and industrial control, direct training in the actual environment and generating data samples in an online manner is often unrealistic or even harmful. Offline RL~\citep{lange2012batch,levine2020offline}, which leverages dataset collected in the past for learning, thus becomes important. Yet, pre-collected data suffers from either noisy simulation or out-of-date measurements, and causes the so-called {\it Sim-to-real} problem~\citep{peng2018sim,zhao2020sim}.

Distributionally robust RL~\citep{iyengar2005robust} was proposed to address the aforementioned issue by learning an optimal policy 
under the worst environment in an {\it uncertainty set}. Robust RL has been studied under Markov decision processes (MDPs) with known transition probabilities in~\citet{xu2010distributionally,wiesemann2013robust,yu2015distributionally}. It has also been studied for MDPs with unknown transitions~\citep{panaganti2022sample,wang2021online,dong2022online,panaganti2022robust,blanchet2023double}. 
Regarding robust RL under non-Markovian decision processes, only robust POMDP with known transitions in the uncertainty set has been studied in ~\citet{osogami2015robust,rasouli2018robust,nakao2021distributionally}. To the best of our knowledge, there has not been any study on robust non-Markovian RL with unknown transitions.

Several challenges need to be handled for exploring the robust offline non-Markovian RL with unknown transitions. (i) Considering non-Markovian processes in the entire uncertainty set can cause dramatic increase of sample complexity. The challenge then lies in how to exploit the structure conditions of these non-Markovian processes for sample efficient design. (ii) Existing coverage condition for offline non-Markovian RL \citep{huang2023provably} is strong and can incur more stringent coverage conditions for {\it robust offline RL}. It is thus critical to relax such a condition to design desirable offline algorithms. These challenges pose the following fundamental open question for us to address:

{ \it Q: Can we design a provably efficient algorithm for learning {\bf robust non-Markovian} decision processes using {\bf offline data}, which (i) effectively exploit the structure of non-Markovian decision processes, and (ii) relaxes the coverage condition for an offline dataset?}

{\bf Our contributions.} In this paper, we provide the {\bf first study} of robust non-Markovian RL with unknown transitions, i.e., unknown uncertainty set. We focus on the setting where an offline dataset is provided and is collected from one (unkown) {\it nominal} environment in an uncertainty set, and the environments are non-Markovian. We summarize our contributions as follows.
\begin{itemize}[topsep=0pt, leftmargin=10pt]

\item  {\bf New algorithm design.}  We propose a novel algorithm for robust offline RL when the nominal model admits a low-rank structure of predictive state representations (PSRs).  In particular, we consider two types of uncertainty sets, namely $\Tc$-type and $\Pc$-type, for non-Markovian decision processes. Our algorithm features two new elements: (i) a novel {\it dataset distillation,} where the distilled dataset ensures that the estimated model from vanilla maximum likelihood estimation performs well on the distilled data; and (ii) {\it lower confidence bound design for the robust value}, which is derived by leveraging the low-rank structure and the newly developed dual forms for robust values under both $\Tc$-type and $\Pc$-type uncertainty sets.
\vspace{-0.05in}

\item {\bf Novel coverage condition and theory for low-rank non-Markovian RL.} We characterize the near-optimality guarantee of our algorithm based on novel {\it type-I concentrability coefficient} for offline learning in low-rank non-Markovian decision processes. This coefficient is much more relaxed than that proposed for offline low-rank non-Markovian RL in~\citet{huang2023provably}. This relaxation is due to refined analysis on the distributional shift. We also identify a wellness condition number of the nominal model. Our results suggests that as long as the concentrability coefficient and the condition number are finite, our algorithm can find a robust policy with suboptimal gap vanishing at a rate of $O(1/\sqrt{N})$, where $N$ is the number of offline samples.
\vspace{-0.05in}

\item {\bf Algorithm and theory for general non-Markovian RL.} We develop a robust algorithm for general offline non-Markovian RL where the transition probabilities may not have any structure. We show that our algorithm can learn a near-optimal robust policy efficiently provided that the offline dataset has finite {\it Type-II concentrability coefficient}, which is also a much weaker condition than that proposed in \citet{huang2023provably}.
\end{itemize}

\vspace{-3mm}
\section{Related Work}
\vspace{-0.05in}

{\bf Offline RL.} 
Offline RL problem has been investigated extensively in the literature \citep{antos2008learning,bertsekas2011approximate,lange2012batch,chen2019information,xie2020q,levine2020offline,xie2021bellman}. Many methods such as Q-learning~\citep{jin2021pessimism}, its variant Fitted Q-Iteration (FQI)~\citep{gordon1995stable,ernst2005tree,munos2008finite,farahmand2010error,lazaric2012finite,liu2020provably,xie2021bellman}, value iteration~\citep{nguyen2023instance,xiong2022nearly}, policy gradient~\citep{zhou2017end}, and model-based approaches~\citep{uehara2021pessimistic,uehara2021representation} have been shown to be effective in offline RL either theoretically or empirically.   
Among them, the pessimism principle is a popular approach to handle the distribution shift caused by an offline policy. Empirically, pessimistic approaches perform well on simulation control tasks \citep{kidambi2020morel,yu2020mopo,kumar2020conservative,liu2020provably,chang2021mitigating}. 
On the theoretical side, pessimism allows us to obtain the Probably Approximately Correct (PAC) guarantee on various models under a coverage condition~\citep{jin2021pessimism,yin2021near,rashidinejad2021bridging,zanette2021provable,uehara2021pessimistic,uehara2021representation}. In this work, we also adopt the pessimism principle and is the first to establish its efficiency for robust offline non-Markovian RL. We develop new coverage conditions in order to guarantee the provable efficiency.

{\bf Robust RL.} Robust MDP (RMDP) was first introduced by~\citet{iyengar2005robust,nilim2005robust}. 
Most of RMDP research in the literature focuses on the {\it planning problem}~\citep{xu2010distributionally,wiesemann2013robust,tamar2014scaling,yu2015distributionally,mannor2016robust,petrik2019beyond,wang2023robust,wang2023model}, providing computationally efficient algorithms. Based on the data generation process, studies on the {\it learning problem} in RMDP can be primarily divided into three categories. RMDP with generative model has been studied in~\citet{zhou2021finite,yang2022toward,panaganti2022sample,shi2024curious}, 
where a generative model can uniformly collect the offline data corresponding to each and every state-action pair.
Online RMDP has been studied in~\citet{roy2017reinforcement,lim2013reinforcement,wang2021online,dong2022online}, which allows allows the agent adaptively sample states and actions in the nominal model. Offline RMDP has been studied in~\citet{panaganti2022robust,blanchet2023double,shi2024distributionally,wang2024sample}, where only a pre-collected dataset is provided, which makes the problem more difficult due to distributional shift~\citep{levine2020offline}.

In addition to RMDP, several works have studied robust POMDPs with assumptions such as known transitions in the uncertainty set~\citep{osogami2015robust,rasouli2018robust,nakao2021distributionally}. In particular, they focus on the planning problem in robust POMDPs, i.e., finding an optimal policy with full information of the uncertainty set, while our work focus on the learning problem that characterizes the sample complexity 
with unknown uncertainty sets.

{\bf Non-Markovian RL.} Learning in non-Markovian decision processes is both computationally and statistically intractable in general~\citep{mossel2005learning, mundhenk2000complexity, papadimitriou1987complexity,vlassis2012computational,krishnamurthy2016pac}. Recently, a line of research \citep{boots2011closing,jiang2018completing,zhang2022reinforcement,zhan2022pac,uehara2022provably,liu2022optimistic,chen2022partially,zhong2022posterior,huang2023provably} has obtained polynomial sample complexity with various structural conditions.  Among them, low-rank structure~\citep{uehara2022provably,liu2022optimistic} has been proved to admit predictive state representations (PSR)~\citep{littman2001predictive,singh2012predictive}, and captures and
generalizes a rich subclass of non-Markov decision-making problems such as MDPs and POMDPs. However, none of these works studies robust offline non-Markovian RL.

\vspace{-3mm}
\section{Preliminaries}\label{sec:pre}
\vspace{-0.05in}
{\bf Notations.}  For a vector $x$, $\|x\|_A$ stands for $\sqrt{x^{\top}Ax}$, and its $i$-th coordinate is represented as $[x]_i$. For functions $\mathbb{P}$ and $\mathbb{Q}$ (not necessarily probability measures) over a set $\mathcal{X}$, the $\ell_1$ distance between them is $ \left\|\Pb -\Qb \right\|_1 = \sum_x|\Pb(x)-\Qb(x)|$, while the hellinger-squared distance is defined as $\mathtt{D}_{\mathtt{H}}^2 (\Pb(x),\Qb(x)) = \frac{1}{2}\sum_x (\sqrt{\Pb(x)} - \sqrt{\Qb(x)} )^2$. For a sequence \( \{x_n, x_{n+1}, \ldots, x_m\} \) with \( n < m \), we denote it shortly as \( x_{n:m} \). For a sequence of sets \( \mathcal{S}_1, \mathcal{S}_2, \ldots, \mathcal{S}_n \), the Cartesian product is represented by \( \mathcal{S}_1 \times \mathcal{S}_2 \times \cdots \times \mathcal{S}_n \). When these sets are identical, this product simplifies to \( \mathcal{S}_1^n \). $\Delta(\mathcal{S})$ consists of all probability distributions over the set $\mathcal{S}$.

\vspace{-3mm}
\subsection{Non-Markovian Decision Processes} 

In this work, we consider finite-horizon episodic (non-Markovian) decision processes with an observation space $\mathcal{O}$, an action space $\mathcal{A}$ and a horizon $H$. The dynamics of the process is determined by a set of distributions $ \{ \Tc_h(\cdot|o_{1:h}, a_{1:h}) \in \Delta(\mathcal{O}) \}_{h=1}^{H-1}$, where $o_{1:h}\in\mathcal{O}^h$ denotes a sequence of observations up to time $h$, $ a_{1:h}\in\mathcal{A}^{h}$ denotes a sequence of actions up to time $h$, and $\Tc_h(o_{h+1} |o_{1:h},a_{1:h})$ denotes the probability of observing $o_{h+1}$ given history $o_{1:h},a_{1:h}$.  A policy $\pi$ is defined by a collection of distributions $\pi = \{\pi_h: \mathcal{O}^{h}\times\mathcal{A}^{h-1} \rightarrow \Delta(\mathcal{A}) \}_{h=1}^H$ chosen by an agent, with each $\pi_h(\cdot|\tau_{h-1},o_h)$ being a distribution over $\mathcal{A}$.  
Specifically, at the beginning of each episode, the process starts at a fixed observation $o_1\in\mathcal{O}$ at time step 1.
After observing $o_1$, the agent samples (takes) an action $a_1\in\mathcal{A}$ according to the distribution (policy) $\pi_1(\cdot|o_1)$, and the process transits to $o_2\in\mathcal{O}$, which is sampled according to the distribution $\Tc_1(\cdot|o_1,a_1)$. 
Then, at any time step $h\geq 2$, due to the non-Markovian nature, the agent samples an action $a_h$ based on all past information $(o_{1:h}, a_{1:h-1})$, and the process transits to $o_{h+1}$, sampled from $\Tc_h(o_{h+1}|o_{1:h}, a_{1:h})$. The interaction terminates after time step $H$.

Notably, the general non-Markovian decision-making problem subsumes not only fully observable MDPs but also POMDPs as special cases. In MDPs, $\Tc_{h-1}(o_h|\tau_{h-1}) = \Tc_{h-1}(o_h|o_{h-1},a_{h-1})$, and in POMDPs, $\Tc_{h-1}(o_h|\tau_{h-1})$ can be factorized as $\Tc_{h-1}(o_h|\tau_{h-1}) = \sum_s\mathbb{O}_h(o_h|s_h)\mathbb{T}_{h-1}(s_h|s_{h-1},a_{h-1})\cdots\mathbb{T}_1(s_2|s_1,a_1)$, where $\{s_t\}_{t=1}^H$ represents unobserved states, $\{\mathbb{O}_t\}_{t=1}^H$ and $\{\mathbb{T}_t\}_{t=1}^H$ are called emission distributions and transition distributions, respectively.

We note that an alternative equivalent description of the process dynamics is the set of probabilities $\{\Pc(o_{1:H}|a_{1:H-1}) = \Tc_{H-1}(o_H|o_{1:H-1}, a_{1:H-1})\cdots\Tc_1(o_2|o_1,a_1), \forall o_{1:H},a_{1:H-1} \}$,  which is also frequently used in this paper. In fact, both $\Pc$ and $\Tc$ fully determine the dynamics with an additional constraint for $\Pc$, that is, $\sum_{o_{h+1:H}}\Pc(o_{1:H}|a_{1:H-1})$ is constant with respect to $a_{h:H-1}$ for all action sequence $a_{h:H-1}$. Hence, we use $\mathbb{D} = \{\Tc_h,\Pc\}$ to represent the model dynamics. To make the presentation more explicit, we introduce a unified parameterization $\theta\in \Theta\subset \mathbb{R}^D$ for the process dynamics, written as $\mathbb{D}_{\theta} = \{\Tc_h^{\theta}, \Pc^{\theta}\}$. Here, $\Theta$ is the complete set of model parameters and $D$ is usually a large integer. We also use $\theta$ to refer to a model.

For a given reward function $R:(\mathcal{O}\times\mathcal{A})^H\rightarrow [0,1]$, and any policy $\pi$, the value of the policy $\pi$ under the dynamics $\mathbb{D}_{\theta}$ is denoted by $V_{\theta, R}^{\pi} = \Eb_{\tau_H\sim \Db_{\theta}^{\pi} } [R(o_{1:H},a_{1:H})]$. We remark here that the goal of an agent in standard RL is to find an $\epsilon$-optimal policy $\hat{\pi}$ such that $\max_{\pi}V_{\theta,R}^{\pi} - V_{\theta,R}^{\hat{\pi}} \leq \epsilon$ for some small error $\epsilon>0$ under a single model $\theta$.

\vspace{-2mm}
\subsection{Robust RL} While standard RL assumes that the system dynamics remains the same, robust RL diverges from this setting by focusing on the `worst-case' or {\it robust value} in situations where the parameter $\theta$ can change within a defined uncertainty set $B\subset \mathbb{R}^D$ (where $B$ will be defined below). Mathematically, given a policy $\pi$, the robust value under the uncertainty set $B$ is defined by
$V_{B, R}^\pi = \min_{\theta\in B} V_{\theta, R}^{\pi}$.

In this paper, we consider two types of uncertainty sets defined respectively on $\Tc$ and $\Pc$ via $f$-divergence \(\mathtt{D}_f(P, P') = \int f(dP/dP')dP'\), where $f$ is a convex function for measuring the uncertainty.  Given a {\it nominal} model $\theta^*$,  a $\Tc$-type uncertainty set under $f$-divergence with center $\theta^*$ and radius 
$\xi$ is defined as 
\begin{align}
&B_{\Tc}^f(\theta^*) :=   \left\{ \theta: \mathtt{D}_f \left(\Tc_h^{\theta}(\cdot|\tau_h), \Tc_h^{\theta^*}(\cdot|\tau_h) \right) \leq \xi, \forall h, \tau_h \right\}. \label{eqn: T uncertain set}
\end{align}

Similarly, a $\Pc$-type uncertainty set under $f$-divergence with center $\theta^*$ and radius $\xi$ is defined as  
\begin{align} 
&B_{\Pc}^f(\theta^*) :=   \left\{ \theta:  \mathtt{D}_{f}\left(\Pc^{\theta}(\cdot|a_{1:H}), \Pc^{\theta^*}(\cdot|a_{1:H})\right) \leq \xi, \forall a_{1:H} \right\}. \label{eqn: P uncertain set}
\end{align}

For simplicity, we set $\xi$ to be a constant throughout the paper. Our proposed algorithm and analysis can be easily extended to the case when $\xi$ is a function of $\tau_h$ or $a_{1:H}$. Notably, the $\Tc$-type uncertainty set reduces to the $\mathcal{S}\times \mathcal{A}$-rectangular uncertainty set~\citep{iyengar2005robust} if $\Db$ reduces to a Markov decision process, i.e., $\Tc_h(\cdot|\tau_h) = \Tc_h(\cdot|o_h,a_h)$. However, due to the non-Markovian nature of our problem, we also consider the $\Pc$-type uncertainty set as $\Pc$ also fully determines the process. 

We study two example types of $f$-divergence, which corresponds to two classical divergences. (i) If $f_1(t) = |t-1|/2$, then $\mathtt{D}_{f_1}$ is called {\it total variation distance}; (ii) If $f_2(t) = t\log t$, then $\mathtt{D}_{f_2}$ is called {\it Kullback–Leibler (KL) divergence}.  We also use $B^i$ to represent $B^{f_i}$ for simplicity.

We focus on  {\bf robust offline RL}, where the agent is provided with an offline dataset $\mathcal{D} = \{\tau_H^n\}_{n=1}^N$ that contains $N$ sample trajectories collected via a {\it behavior} policy $\rho$ under the nominal model $\theta^*$. Given the type of the uncertain set $B\in\{B_{\Tc}^i, B_{\Pc}^i\}_{i=1}^2$ (without knowing the center $\theta^*$), the learning goal of the agent is to find a near-optimal robust policy $\hat{\pi}$ using this offline dataset $\mathcal{D}$ such that
\[\max_{\pi}\min_{\theta\in B} V_{\theta, R}^{\pi}  - \min_{\theta\in B} V_{\theta, R}^{\hat{\pi}} \leq \epsilon. \]

\subsection{Low-rank Non-Markovian Decision Process }\label{sec: psr define}

We first introduce some additional notations that are helpful to describe low-rank non-Markovian decision processes. Given a historical trajectory at time step $h$ as $\tau_h:= (o_{1:h},a_{1:h})$, we denote a future trajectory as $\omega_h:= (o_{h+1:H},a_{h+1:H})$. The set of all $\tau_h$ is denoted by $\mathcal{H}_h = (\mathcal{O}\times\mathcal{A})^{h}$ and the set of all future trajectories is denoted by $\Omega_h = (\mathcal{O}\times\mathcal{A})^{H-h}$. In addition, we add a superscript $o$ or $\mathrm{a}$ to either historical trajectory $\tau$ or a future trajectory to represent the observation sequence and the action sequence contained in that trajectory, respectively. For example, $\omega_h^o = o_{h+1:H}$ and $\omega_h^a = a_{h+1:H}$. To better describe a policy, we also use the notation $x_h = (\tau_{h-1},o_h) = (o_{1:h},a_{1:h-1})$ as the available information for the learning agent at time step $h$. We remark that, having $\tau_h, x_h$ in the same equation implies that $\tau_h = (x_h,a_h).$

Recall that a non-Markov decision process $\theta$ can be fully determined by $\Pc^{\theta}$. We introduce a collection of dynamic matrices $\{\mathbb{M}_h^{\theta} \in \mathbb{R}^{|\mathcal{H}_h|\times |\Omega_h|} \}_{h=1}^H$. The entry at the $\tau_h$-th row and $\omega_h$-th column of $\mathbb{M}_h^{\theta}$ is $\Pc^{\theta}(\omega_h^o, \tau_h^o | \tau_h^a, \omega_h^a)$.

\begin{definition}[Rank-$r$ non-Markov decision process]
    A non-Markov decision process $\theta$ is rank $r$, if for any $h$, the model dynamic matrix $\mathbb{M}_h^{\theta}$ has rank $r$. 
\end{definition}

 It has been proved that any low-rank sequential decision-making problem $\theta$ admits Predictive State Representations (PSR) \citep{liu2022optimistic} if for each $h$, there exists a set of future trajectories, namely, core tests known to the agent, $\mathcal{Q}_h = \{\mathbf{q}_{h,1},\ldots, \mathbf{q}_{h,d}\} \subset \Omega_h$,
such that the submatrix restricted to these tests (columns) $\mathbb{M}_h^{\theta}[\mathcal{Q}_h]$ has rank $r$, where $d\geq r$ is a positive integer\footnote{Setting the same $d$ for different $h$ is for notation simplicity.}. 
We consider all low-rank problems with the same set of core tests $\{\mathcal{Q}_h\}_h^H$, and have the following representations. 
\begin{align*}
    &\Pc^{\theta}(o_H,\ldots,o_1 | a_1,\dots, a_H) = \mbf(\omega_h)^{\top} \psi(\tau_h), \quad \phi_h = \sum_{o_H,\ldots,o_{h+1}}\mbf(\omega_h).
\end{align*}
where $\mbf(\omega_h)\in\mathbb{R}^{d }$, and $\psi(\tau_h)\in\mathbb{R}^{d }$. We denote  $\theta=(\phi_h,\psi(\tau_h))$ to be the representation.
In particular, $\Theta$ is the set of all low-rank problems with the same core tests $\{\mathcal{Q}_h\}_{h=1}^H$.

We assume $\psi_0$ is known to the agent~\citep{huang2023provably}, and employ a {\it bar} over a feature $\psi(\tau_h)$ to denote its normalization $\bar{\psi}(\tau_h) = \psi(\tau_h)/\phi_h^{\top}\psi(\tau_h)$. $\bar{\psi}(\tau_h)$ is also known as the prediction vector \citep{littman2001predictive} or prediction feature of $\tau_h$, since $[\bar{\psi}(\tau_h)]_{\ell}$ is the probability of observing $\mathbf{q}_{h,\ell}^{\mathrm{o}}$ given the history $\tau_h$ and taking action sequence $\mathbf{q}_{h,\ell}^{\mathrm{a}}$.

Moreover, let $\mathcal{Q}_h^A = \{\mathbf{q}_{h,\ell}^{\mathrm{a}}\}_{\ell=1}^{d}$ be the set of action sequences that are part of core tests, constructed by eliminating any repeated action sequence. $\mathcal{Q}_h^A$ known as the set of core action sequences.

We further assume that PSRs studied in this paper are well-conditioned, as specified in \Cref{assmp:well-condition}. Such an assumption and its variants are commonly adopted in the study of PSRs \citep{liu2022optimistic,chen2022partially,zhong2022posterior,huang2023provably}.

\begin{assumption}[$\gamma$-well-conditioned PSR]\label{assmp:well-condition}
A PSR $\theta$ is said to be $\gamma$-well-conditioned if
\begin{align*}
\textstyle    \forall h,~~ \max_{x\in\mathbb{R}^{d_h}:\|x\|_1\leq 1} \max_{\pi,\tau_h} \sum_{\omega_h}\pi(\omega_h|\tau_h)|\mbf(\omega_h)^{\top}x|\leq \frac{1}{\gamma},
\end{align*}
where $\pi(\omega_h|\tau_h) = \pi(a_H|x_H)\cdots\pi(a_{h+1}|x_{h+1}).$
\end{assumption}
\Cref{assmp:well-condition} requires that the error of estimating $\theta$ does not significantly blow up when the estimation error $x$ of estimating the probability of core tests is small.

 \section{Robust Offline RL with Low-rank Nominal Model}\label{sec:main psr} 

In this section, we study the setting when the nominal model has {\em low-rank} structure (with rank $r$)~\citep{zhan2022pac,liu2022optimistic,chen2022partially,huang2023provably}. 
To unify the notations, we use $B_{\mathrm{u}}^i(\theta)$ to denote the uncertainty set, where $\mathrm{u}$ is either symbol $\Tc$ or $\Pc$.

\vspace{-3mm}
\subsection{Algorithm Design}\label{sec:psr alg}

We present our algorithm for robust offline RL when the nominal model enjoys low-rank structure. Our algorithm features {\bf two main novel elements:} (a) dataset distillation, which allows the estimated prediction features of distilled data being close to the true features, and thus enables us to construct (b) a lower confidence bound design for the robust value $V_{B_{\mathrm{u}}^i(\theta^*), R}^{\pi}$. In particular, we develop a new dual form for the robust value $V_{B_{\Pc}^i(\theta), R}^{\pi}$ under $\Pc$-type uncertainty set, that is more amenable to practical implementation. The main steps of the algorithm are explained in detail as follows, and the pseudo-code is presented in \Cref{alg: robust PSR}.

\begin{algorithm}[h]
\caption{Robust Offline RL with low-rank nominal model}
\label{alg: robust PSR}
\begin{algorithmic}[1]
\STATE {\bf Input:} offline dataset $\mathcal{D}$, divergence type $i$
 \STATE Estimate a model $\hat{\theta}$ according to \Cref{eqn: loss}.
 \STATE Obtain distilled dataset $\Dc^{\mathrm{g}}$ according to \Cref{eqn: good dataset}.
 \STATE Construct $\hat{b}$ according to \Cref{eqn: LCB bonus}.
\STATE Output $\hat{\pi}$ such that
$\hat{\pi} = \arg\max_{ \pi \in \Pi}   \left(  V_{ B_{\mathrm{u}}^i(\hat{\theta}), R}^{\pi} - C_{\mathrm{u}}^iV_{\hat{\theta},\hat{b}}^{\pi}\right).$

\end{algorithmic}
\end{algorithm}

{\bf Model estimation and dataset reform.} First, we perform maximum likelihood estimation of the nominal model $\theta^*$ by minimizing the following negative log-likelihood loss function
\begin{align}
\textstyle    \mathcal{L}(\theta|\Dc) =  \frac{1}{N} \sum_{n=1}^N - \log \Db_{\theta}^{\rho}(\tau_H^n), \label{eqn: loss}
\end{align}
which can be done using a supervised learning oracle, e.g. stochastic gradient descent. Let $\hat{\theta} = \arg\min_{\theta} \mathcal{L}(\theta|\Dc)$.

{\bf Novel dataset distillation.} Motivated by~\citet{huang2023provably}, to construct a desirable lower confidence bound for any value function under the nominal model, the estimated model $\hat{\theta}$ needs to assign non-negligible probabilities on every sample trajectory. Differently from \citet{huang2023provably} that introduces additional constraints to the optimization oracle, we propose a novel idea that only keeps the ``good'' sample trajectories whose estimated probability is greater than a pre-defined small value $p_{\min}$ under the estimated model $\hat{\theta}$ and collect those sample trajectories in set $ \Dc^{\mathrm{g}}$, defined as follows.
\begin{align}
    \Dc^{\mathrm{g}} =  \left\{\tau_H \in \Dc,~~ \Db_{\hat{\theta}}^{\rho}(\tau_H) \geq    p_{\min} \right\}. \label{eqn: good dataset}
\end{align}
We prove in \Cref{lemma:good samples enough} that, with high probability  $|\Dc^{\mathrm{g}}| = \Omega(N) $, which still guarantees sufficient number of samples for learning. Then, $\Dc^{\mathrm{g}}$ is randomly and evenly divided into $H$ datasets $\Dc_0^{\mathrm{g}},\ldots, \Dc_{H-1}^{\mathrm{g}}$, in order to separately learn $\bar{\psi}^*(\tau_h)$ for each $h\in\{0,1,\ldots,H-1\}$.

{\bf LCB of robust value. } Given the estimated model $\hat{\theta}$ and divergence type $i$, \Cref{alg: robust PSR} constructs a lower confidence bound (LCB) of $V_{B_{\mathrm{u}}^i(\theta^*), R}^{\pi}$ for all policy $\pi$ through a bonus function $\hat{b}$ defined as:  
\begin{align}\label{eqn: LCB bonus}
    &\textstyle \hat{U}_{h}  = \lambda I + \sum_{\tau_h\in\Dc_h^{\mathrm{g}} } \bar{\hat{\psi}} (\tau_h) \bar{\hat{\psi}} (\tau_h)^{\top},\qquad
    \hat{b} (\tau_H) = \min\Big\{ \alpha \sqrt{ \sum_h \left\| \bar{\hat{\psi}} (\tau_h) \right\|^2_{(\hat{U}_{h} )^{-1}}}, 1 \Big\},
\end{align}
where $ \lambda $ and  $\alpha_i$ are pre-defined parameters. Recall that $\bar{\hat{\psi}}(\tau_h)$ is the prediction feature described in \Cref{sec: psr define}. Then, we can show that $V_{ B_{\mathrm{u}}^i(\hat{\theta}), R}^{\pi} - C_{\mathrm{u}}^iV_{\hat{\theta},\hat{b}}^{\pi}$ is an LCB of $V_{B_{\mathrm{u}}^i(\theta^*), R}^{\pi}$, where $C_{\mathrm{u}}^i$ (defined in \Cref{sec:psr thm}) is a {\it scaling parameter} that depends on the type of uncertainty set.

{\bf Novel dual form of robust value.} For the $\Tc$-type uncertainty sets, define robust value function as:
\begin{align*}
    &\textstyle V_{B(\theta^*),R,h}^{\pi}(x_h) = \inf_{\theta\in B(\theta^*)}\Eb_{\theta}^{\pi} \left[  \sum_{h' = h}^H R(\btau_{h'}) \bigg| x_h\right],\\
    &\textstyle Q_{B(\theta^*),R,h}^{\pi}(\tau_{h}) = \inf_{\theta\in B(\theta^*)}\Eb_{\theta}^{\pi} \left[  \sum_{h' = h}^H R(\btau_{h'}) \bigg| \tau_{h}\right],
\end{align*}
where $\tau_h = (x_h,a_h)$. Then, the following Bellman-type equations hold for robust value functions.
\begin{lemma}\label{lemma:bellman main paper}
    For any $\theta$, let $V_{B(\theta), R,H+1}^{\pi}(\tau_H) =0$. Then, we have
\begin{align*}
        &\textstyle V_{B(\theta),R,h}^{\pi}(x_h) = \sum_{a_h}\pi(a_h|x_h)Q_{B(\theta),R,h}^{\pi}(\tau_h)\\
        &\textstyle Q_{B(\theta),R,h}^{\pi}(\tau_h) = R(x_h) + \inf_{\theta'\in B(\theta^*)} \sum_{o_{h+1}}\Tc_h^{\theta'}(o_{h+1}|\tau_h) V_{B(\theta), R, h+1}^{\pi}(x_{h+1}).
\end{align*}
\end{lemma}
Therefore, robust value under $\Tc$-type uncertainty set can be calculated using the recursive formula above. Notably, the calculation of $Q_{B(\theta),R,h}^{\pi}(\tau_h)$ involves solving an optimization problem in the form of $\inf_{P: \mathtt{D}_{f}(P, P^*)}\Eb_P[g(X)]$, where $P$ and $P^*$ are two probability distributions over a set $\mathcal{X}$. This optimization problem can be solved efficiently as shown in the literature~\citep{panaganti2022robust,blanchet2023double}. We also provide its dual form in \Cref{sec:tech lemmas}.

On the other hand, when the uncertainty set is $\Pc$-type, we develop a new dual form of the robust value $V_{B_{\Pc}^1(\theta^*), R}^{\pi}$ and $V_{B_{\Pc}^2(\theta^*), R}^{\pi}$ in \Cref{eqn: P type TV dual} and \Cref{eqn: P type KL dual}, respectively. Note that, in those equations, $f(x_H) = \sum_{a_H}R(\tau_H)\pi(\tau_{H-1})$. 
\begin{align} 
    V_{B_{\Pc}^1(\theta^*), R}^{\pi} =  \max_{ \boldsymbol{\gamma}\geq 0, \boldsymbol{\lambda}  } \Big\{&   \sum_{x_H}(f(x_H) - \gamma_{x_H})\Pc^{\theta^*}(x_H)   - \sum_{a_{1:H-1}} \max_{o_{1:H}} \left|\gamma_{x_H} - f(x_H) - \lambda_{a_{H-1},x_{H-1}}\right| \xi \Big\} \label{eqn: P type TV dual}\\
    V_{B_{\Pc}^2(\theta^*), R}^{\pi} = \max_{\boldsymbol{\eta}\geq 0 } \Big\{&  -\sum_{\tau_{H-1}}\eta_{a_{1:H-1}}\Pc^{\theta^*}(\tau_{H-1})\log {\Eb}_{o_H\sim \Tc^{\theta^*}_{H-1}}\left[ e^{ - \frac{f(x_H)}{ \eta_{a_{1:H-1}}} } \right]   - \sum_{a_{1:H-1}}\eta_{a_{1:H-1}}\xi\Big\}\label{eqn: P type KL dual}
\end{align}

{\bf Output policy design.} Finally, \Cref{alg: robust PSR} outputs a policy $\hat{\pi} = \arg\max_{ \pi \in \Pi}   \left( V_{ B_{\mathrm{u}}^i(\hat{\theta}), R}^{\pi} - C_{\mathrm{u}}^iV_{\hat{\theta},\hat{b}}^{\pi} \right)$, which maximizes a {\it lower confidence bound} of $V_{B_{\mathrm{u}}^i(\theta^*),R}^{\pi}$.

\vspace{-3mm}
\subsection{Theoretical Guarantee}\label{sec:psr thm}

In this section, we present the theoretical results for \Cref{alg: robust PSR}. Before we state the theorem, we need to introduce several definitions and condition coefficients.

{\bf Bracketing number.} First, we introduce a complexity measure of the model class $\Theta$, which is standard in the literature~\citep{liu2022optimistic,huang2023provably}.

\begin{definition}[$\varepsilon$-Bracketing number of $\Theta$]\label{def:optimistic net}
Suppose for each $\theta\in\Theta$, there exists a function $F_{\theta}:\mathcal{X}\rightarrow\mathbb{R}_+$ parameterized by $\theta$.
   Given two functions $l$ and $g: \mathcal{X} \to \mathbb{R}_{+}$, the bracket $[l,g]$ is the set of all $\theta\in\Theta$ satisfying $l\leq F_{\theta}\leq g$.  An $\varepsilon$-bracket is a bracket $[l,g]$ with $\|g-l\|_1 < \eta$. The bracketing number $\mathcal{N}_{\varepsilon}(\Theta)$ is the minimum number of $\eta$-brackets needed to cover $\Theta$.
\end{definition}

Note that if $\Theta$ is the parameter space for tabular PSRs with rank $r$, we have $\log\mathcal{N}_{\varepsilon}(\Theta)\leq r^2|\mathcal{O}||\mathcal{A}|H^2\log\frac{H|\mathcal{O}||\mathcal{A}|}{\varepsilon}$ (see Theorem 4.7 in \citet{liu2022optimistic}).

{\bf Novel concentrability coefficient.} 
Under the low-rank structure and PSRs, we introduce the following novel concentrability coefficient $C_{\mathrm{p}}(\pi|\rho)$, namely {\bf Type-I concentrability coefficient}, of a behavior policy $\rho$ against a robust policy $\pi$.
\begin{align*}
\textstyle    C_{\mathrm{p}}(\pi|\rho)  = \max_h \max_{x} \frac{ x^{\top} \mathop{\Eb}_{\tau_h\sim \Db_{\theta^*}^{\pi}}[ \bar{\psi}^*(\tau_h)\bar{\psi}^*(\tau_h)^{\top}  ] x }{ x^{\top} \mathop{\Eb}_{\tau_h\sim \Db_{\theta^*}^{\rho }}[ \bar{\psi}^*(\tau_h)\bar{\psi}^*(\tau_h)^{\top}  ] x }.
\end{align*}

\begin{remark}
    Requiring a finite Type-I concentrability coefficient is much weaker than requiring a finite concentrability coefficient $C$ defined in~\citet{huang2023provably}, in which a point-to-point probability ratio is used, i.e., $C \geq \Db_{\theta^*}^{\pi}(\tau_h)/ \Db_{\theta^*}^{\rho}(\tau_h)$. It is obvious that $C_{\mathrm{p}}(\pi|\rho)\leq C$. Moreover, even if $\Db_{\theta^*}^{\rho}(\tau_h)$ is exponentially small, indicating $C$ is exponentially large, the minimum eigenvalue of the covariance matrix of $\bar{\psi}^*(\tau_h)$ under $\rho$ (i.e., $\mathop{\Eb}_{\tau_h\sim \Db_{\theta^*}^{\rho }}[ \bar{\psi}^*(\tau_h)\bar{\psi}^*(\tau_h)^{\top}  ]$) could be bounded way from zero. By \Cref{assmp:well-condition}, we have $\|\bar{\psi}^*(\tau_h)\|_1 /\gamma \geq \sum_{\omega_h}\pi(\omega_h|\tau_h)\mathbf{m}(\omega_h)^{\top}\bar{\psi}^*(\tau_h) = 1$, which implies that the inverse of the norm of the prediction feature scales in Poly$(d,1/\gamma)$. Thus, with a good behavior policy $\rho$, the scaling of $C_{\mathrm{p}}(\pi|\rho)$ is likely to be polynomial.
\end{remark}

{\bf Wellness condition of the nominal model.} When the uncertainty set is $\Tc$-type, we further introduce a wellness condition number $C_B$ of the nominal model $\theta^*$ defined as follows.
\[\textstyle C_B = \max_{\theta\in B(\theta^*) }   \max_h  \frac{ \Db_{\theta}^{\pi}(\tau_h) }{\Db_{\theta^*}^{\pi}(\tau_h)}  =\max_{\theta\in B(\theta^*) } \max_h \frac{ \Pc^{\theta}(\tau_h) }{ \Pc^{\theta^*}(\tau_h) }.  \]

Intuitively, this condition number is called the maximum likelihood ratio between the distributions $\Pc^{\theta^*}(\cdot|a_{1:h})$ and $\Pc^{\theta}(\cdot|a_{1:h})$ and only depends on the property of the nominal model and the shape of the uncertainty set. In particular, when $B= B_{\Tc}^2$, we must have 
\begin{align*}
\textstyle    \Eb_{\theta}^{\pi}\left[\log \frac{\Pc^{\theta}(\tau_h)}{ \Pc^{\theta^*}(\tau_h)} \right] = \Eb_{\theta}^{\pi}\left[\sum_{h'\leq h}\log \frac{\Tc^{\theta}(o_{h'}|\tau_{h'-1})}{ \Tc^{\theta^*}(o_{h'}|\tau_{h'-1})} \right]\leq \xi h,
\end{align*}
which should hold for all policy $\pi$, and thus implies that $C_B$ is highly likely to be small.

{\bf Scaling parameters.}  Finally, we briefly introduce the scaling parameter $C_{\mathrm{u}}^i$ that is used to construct the LCB of robust value. Equipped with the dual forms, we define $C_{\Pc}^1=1$, $C_{\Tc}^1 = C_B$. When uncertainty set is determined by KL divergence, we make the following assumption.

\begin{assumption}\label{assm:dual P eta}
     Suppose that $\boldsymbol{\eta}^*\in\mathbb{R}^{|\mathcal{A}|^{H-1}}$ is an optimal solution of \Cref{eqn: P type KL dual}, and $\boldsymbol{\eta}^*_{\hat{\theta}}\in\mathbb{R}^{|\mathcal{A}|^{H-1}}$ is an optimal solution of \Cref{eqn: P type KL dual} when $\theta^*$ is replaced by $\hat{\theta}$. We assume that the agent has an estimate of a lower bound $\eta^*$ of the coordinates of both $\boldsymbol{\eta}^*$  and $\boldsymbol{\eta}^*_{\hat{\theta}}$, i.e., $\eta^*\leq \min_{a_{1:H-1}}\{\eta_{a_{1:H-1}}^*, \eta_{\hat{\theta},a_{1:H-1}}^*\}$.
\end{assumption}

\begin{assumption}\label{assm:dual T lambda}
$\lambda_{\theta\|\theta^*}^*(\tau_h)$ is an optimal solution to the following optimization problem:
\begin{align*}
    \sup_{\lambda\geq 0} \left\{ -\lambda \log \Eb_{\Tc_h^{\theta}}[\exp(-V_{B_{\Tc}^2(\theta^*), R, h+1}^{\pi}(\tau_{h},o)/\lambda ) ] - \lambda \xi \right\}.
\end{align*}
We assume that the agent has an estimate of a lower bound $\lambda^*$ of all $\lambda_{\theta^*\|\hat{\theta}}^*(\tau_h)$ and $\lambda_{\hat{\theta}\|\theta^*}^*(\tau_h)$, i.e., 
$\lambda^*\leq \min_{h,\tau_h}\{\lambda_{\theta^*\|\hat{\theta}}^*(\tau_h), \lambda_{\hat{\theta}\|\theta^*}^*(\tau_h)\}.$
\end{assumption}

\begin{remark}
    The optimization problem mentioned in \Cref{assm:dual T lambda} is the dual problem that need to solve in the Bellman-type equation (\Cref{lemma:bellman main paper}). In the algorithm implementation, it suffices to estimate $\min_{a_{1:H-1}}\{\eta^*_{\hat{\theta},a_{1:H-1}}\}$ and $\min_{h,\tau_h}\{\lambda_{\hat{\theta}\|\theta^*}^*(\tau_h)\}$, which is not difficult to obtain such value since $\hat{\theta}$ is known. We make these two assumptions for notation simplicity. Such assumptions also appear in \citet{blanchet2023double}.
\end{remark}

Define the scaling parameters $C_{\Pc}^2 = 3\exp(1/\eta^*)$, and $C_{\Tc}^2 = C_B\max\left\{\exp(\xi)/\xi, \lambda^*\exp(1/\lambda^*) \right\}.$  

We are now ready to present the theorem characterizing the performance of \Cref{alg: robust PSR}.

\begin{theorem}\label{thm:robust PSR}
    Suppose \Cref{assmp:well-condition}, \Cref{assm:dual P eta}, and \Cref{assm:dual T lambda} hold. Let  $Q_A = \max_h|\mathcal{Q}_h^A|$ be the maximum number of core action sequences. Let $\iota = \min_{\mathbf{q}_{h,\ell}^{\mathrm{a}},\tau_h} \rho(\mathbf{q}_{h,\ell}^{\mathrm{a}}|\tau_h) $, $p_{\min} = \frac{\delta}{N(|\mathcal{O}||\mathcal{A}|)^{2H}}$, $\varepsilon= \frac{p_{\min}}{NH}$, $ \beta = O(\log(\mathcal{N}_{\varepsilon}(\Theta)/\delta))$, $ \lambda = H^2Q_A^2$, and $ \alpha  = O\left(\frac{  Q_A\sqrt{dH} }{\gamma^2}\sqrt{\lambda} + \frac{  \sqrt{\beta} }{ \iota \gamma} \right)$. Then, with probability at least $1-\delta$, we have the following four results for the suboptimal gap of the policy $\hat{\pi}$ output by \Cref{alg: robust PSR} under four settings.
    
    (i) $\Pc$-type uncertainty set under TV distance: $
    V_{B_{\Pc}^1(\theta^*), R}^{\pi^*} - V_{B_{\Pc}^1(\theta^*), R}^{ \hat{\pi}} \lesssim C_0 \sqrt{ \frac{C_{\mathrm{p}}(\pi^*|\rho)}{N} }$,
    
    (ii) $\Pc$-type uncertainty set under KL divergence:
    $V_{B_{\Pc}^2(\theta^*), R}^{\pi^*} - V_{B_{\Pc}^2(\theta^*), R}^{ \hat{\pi}} \lesssim C_0 C_{\Pc}^{\KL}\sqrt{ \frac{C_{\mathrm{p}}(\pi^*|\rho)}{N} }$,
    
    (iii) $\Tc$-type uncertainty set under TV distance: 
    $V_{B_{\Tc}^1(\theta^*), R}^{\pi^*} - V_{B_{\Tc}^1(\theta^*), R}^{ \hat{\pi}} \lesssim C_0 C_B\sqrt{ \frac{C_{\mathrm{p}}(\pi^*|\rho)}{N} }$,
    
    (iv) $\Tc$-type uncertainty set under KL divergence,
    $V_{B_{\Tc}^2(\theta^*), R}^{\pi^*} - V_{B_{\Tc}^2(\theta^*), R}^{ \hat{\pi}} \lesssim C_0 C_BC_{\Tc}^{\KL}\sqrt{ \frac{C_{\mathrm{p}}(\pi^*|\rho)}{N} }$,
    
where $\pi^*$ is the optimal robust policy under each setting, $C_{\mathrm{p}}(\pi^*|\rho)$ is the type-I coefficient of the behavior policy $\rho$ against $\pi^*$,  $C_0 = \frac{H^2 Q_A\sqrt{rd\beta}}{\iota \gamma^2} \left(\sqrt{r} + \frac{Q_A\sqrt{H}}{\gamma} \right)$, $C_B$ is the wellness condition number of $\theta^*$,  $C_{\Pc}^{\KL} = 3\exp(1/\eta^*)$, $C_{\Tc}^{\KL} = \max\left\{\exp(\xi)/\xi, \lambda^*\exp(1/\lambda^*) \right\}$, and $\eta^*,\lambda^*$ are defined in \Cref{assm:dual P eta} and \Cref{assm:dual T lambda}.
\end{theorem}

The theorem states that as long as the type-I concentrability coefficient of $\rho$ against the optimal robust policy $\pi^* = \arg\max_{\pi}V_{B(\theta), R}^{\pi}$ is bounded, and the nominal model $\theta^*$ has finite wellness condition number $C_B$, \Cref{alg: robust PSR} finds an $\epsilon$-optimal robust policy provided that the offline dataset has $\Omega(1/\epsilon^2)$ sample trajectories.

{\it Proof sketch of \Cref{thm:robust PSR}.} The proof consists of three main steps. (i) {\bf A novel MLE guarantee.} We show that the estimator $\hat{\theta}$ produced by {\it vanilla} MLE already performs well on a large portion of offline data samples, namely, $\sum_{ \tau_h \in\Dc_h^{\mathrm{g}} }  \|\Dc_{\hat{\theta}}^{\rho}(\cdot|\tau_h) - \Dc_{\theta^*}^{\rho}( \cdot | \tau_h) \|_1\leq O(\beta)$. This result is different from the MLE guarantee proved in~\citet{huang2023provably}, where a constrained MLE oracle is required. (ii) {\bf Simulation lemma for robust value.} By utilizing the newly proposed dual forms in \Cref{eqn: P type TV dual} and \Cref{eqn: P type KL dual}, we show that for any model $\theta$ (not necessarily low-rank), the difference $|V_{B(\theta^*), R}^{\pi} - V_{B(\theta), R}^{\pi}|$ is generally upper bounded by $O(c_B\|\Db_{\theta^*}^{\pi} - \Db_{\theta}^{\pi}\|_1)$, where $c_B$ depends on the uncertainty type. (iii) {\bf Applying concentrability coefficient.} In low-rank non-Markovian decision processes, $\ell_1$ distance $\|\Db_{\theta^*}^{\pi} - \Db_{\theta}^{\pi}\|_1$ can be upper bounded by $\Eb_{\theta^*}^{\pi}[\|\bar{\psi}^*(\tau_h)\|_{\Lambda^{-1}}]$~\citep{liu2022optimistic,huang2023provably}. By Cauchy's inequality, this quantity is determined by the covariance matrix $\Eb_{\theta^*}^{\pi}[\bar{\psi}(\tau_h)\bar{\psi}^*(\tau_h)^{\top}]$, which can thus be controlled by the type-I concentrability coefficient $C_{\mathrm{p}}(\pi|\rho)$ and the covariance matrix under the behavior policy $\rho$. It then yields the final result when combined with the MLE guarantee. The full proof can be found in \Cref{sec:proof of thm1}. 

\section{Robust Offline RL under General Non-Markovian Decision Processes}\label{sec:online PSR}

In this section, we first introduce a more general algorithm (see \Cref{alg: robust}) for robust offline learning in non-Markovian decision processes, and then present its theoretical guarantee.

\subsection{Algorithm Design}\label{sec:gen alg}
The general algorithm utilizes double pessimism for robust offline RL. The algorithm consists of two parts: {\it model estimation} and {\it robust policy planning}. We note that similar double pessimism has been adopted in \citet{blanchet2023double}. However, due to the non-Markovian nature, our robust value and its dual form as well as the corresponding analysis are different from theirs.

{\bf Model estimation.} We construct the confidence set $\mathcal{C}$ for nominal model $\theta^*$ using offline dataset $\mathcal{D}$. 
\begin{align}
\textstyle    \mathcal{C} = \left\{\theta: \sum_{n=1}^N\Pc^{\theta}(\tau_H^n) \geq \max_{\theta}\sum_{n=1}^N \Pc^{\theta}(\tau_H^n) - \beta \right\}, \label{eqn: confidence set}
\end{align}
where $\beta $ is a pre-defined parameter. This confidence set ensures that $\theta^*\in \mathcal{C}$ and any model $\theta\in\mathcal{C}$ is ``close'' to the nominal model $\theta^*$ under the behavior policy $\rho$.

{\bf Robust policy planning.} Then, we output a policy under the double pessimism principle. Specifically, for each model $\hat{\theta}\in\mathcal{C}$, since it is close to the nominal model $\theta^*$, we examine its robust value function $V_{B(\hat{\theta}), R}^{\pi}$ under a policy $\pi$. By minimizing the robust value $V_{B(\hat{\theta}), R}^{\pi}$ over the confidence set $\mathcal{C}$, we find a {\it lower confidence bound} for the true robust value $V_{B(\theta^*), R}^{\pi}$ due to $\theta^*\in\mathcal{C}$. Finally, we output an optimal policy $\hat{\pi}$ with respect to the LCB of $\theta^*\in\mathcal{C}$. Mathematically, $\hat{\pi} = \arg\max_{ \pi \in \Pi} \min_{ \hat{\theta}  \in \mathcal{C} }   V_{B(\hat{\theta}), R}^{\pi}$

\vspace{-3mm}
\subsection{Theoretical Guarantee}\label{sec:gen thm}

In this section, we present the theoretical guarantees of \Cref{alg: robust}. Since the nominal model does not have specific structure or representations, type-I concentrability coefficient is not suitable for this case. Therefore, we introduce the {\bf Type-II concentrability coefficient} $C(\pi|\rho)$ of a behavior policy $\rho$ against a robust policy $\pi$ defined as $C(\pi|\rho) = \sum_{h=1}^H \Eb_{\theta^*}^{\rho} \left[ \left( \frac{ \Db_{\theta^*}^{\pi}(\tau_h) }{\Db_{\theta^*}^{\rho}(\tau_h)} \right)^2 \right]$.
\begin{remark}
    In general, requiring a finite type-II concentrability coefficient is stronger than requiring a finite type-I concentrability coefficient defined in \Cref{sec:psr thm}. However, it is still a relaxed coverage assumption than that in~\citet{huang2023provably}. This is because in the analysis, we rescale the estimation error across the distribution induced by the behavior policy instead of using a point-wise bound. 
\end{remark}

We are now ready to present the main result for \Cref{alg: robust}.
\begin{theorem}\label{thm:general robust RL}
Let $\varepsilon= \frac{1}{NH}$, $ \beta = O(\log(\mathcal{N}_{\varepsilon}(\Theta)/\delta))$. Then, with probability at least $1-\delta$, we have the following four results for the suboptimal gap of the policy $\hat{\pi}$ output by \Cref{alg: robust}. 

(i) $\Pc$-type uncertainty set under TV distance
$V_{B_{\Pc}^1(\theta^*), R}^{\pi^*} - V_{B_{\Pc}^1(\theta^*), R}^{\hat{\pi}}\lesssim H\sqrt{  \frac{C(\pi^*|\rho) \beta }{N}}, $

(ii) $\Pc$-type uncertainty set under KL divergence:
$V_{B_{\Pc}^2(\theta^*), R}^{\pi^*} - V_{B_{\Pc}^2(\theta^*), R}^{\hat{\pi}}\lesssim HC_{\Pc}^{\KL}\sqrt{  \frac{C(\pi^*|\rho) \beta }{N}},$

(iii) $\Tc$-type uncertainty set under TV distance
$ V_{B_{\Tc}^1(\theta^*), R}^{\pi^*} - V_{B_{\Tc}^1(\theta^*), R}^{\hat{\pi}}\lesssim HC_B\sqrt{  \frac{C(\pi^*|\rho) \beta }{N}},$

(iv) $\Tc$-type uncertainty set under KL divergence
$V_{B_{\Tc}^2(\theta^*), R}^{\pi^*} - V_{B_{\Tc}^2(\theta^*), R}^{\hat{\pi}}\lesssim HC_B C_{\Pc}^{\KL}\sqrt{  \frac{C(\pi^*|\rho) \beta }{N}},$

where $\pi^*$ is the optimal robust policy under each setting, $C(\pi^*|\rho)$  is the type-II coefficient of the behavior policy $\rho$ against $\pi^*$,  $C_B$ is the wellness condition number defined in \Cref{sec:psr thm}, and $C_{\Pc}^{\KL}, C_{\Tc}^{\KL}$ are defined in \Cref{thm:robust PSR}.
\end{theorem}
The theorem suggests that \Cref{alg: robust} is more sample efficient than \Cref{alg: robust PSR}, due to the general double pessimism design and the type-II concentrability coefficient, while \Cref{alg: robust PSR} is more amenable to practical implementation due to low-rank structure and computationally tractable bonus function design and dual forms development.

\vspace{-3mm}
\section{Conclusion}
In this work, we studied robust offline non-Markovian RL with unkown transitions. We developed the first sample efficient algorithm when the nominal model admits a low-rank structure. 
To characterize the theoretical guarantee of our algorithm, we propose a novel type-I concentrability coefficient for PSRs, which can be of independent interest in the study of offline low-rank non-Markovian RL. 
In addition, we extended our algorithm to a more general case when the nominal model does not have any specific structure. With the notion of type-II concentrability coefficient, we proved that the extended algorithm also enjoys sample efficiency under all different types of the uncertainty set.

\bibliography{arxiv_main}
\bibliographystyle{apalike}

\newpage

\appendix

\noindent{\LARGE \bf Appendices} 

\section{Technical Lemmas}\label{sec:tech lemmas}
In this section, we present the technical lemmas for proving our main results. It includes three subsections. First, \Cref{sec:dual form} primirily describes the dual forms for robust values under $\Pc$-type uncertainty sets. Second, \Cref{sec:bellman type} verifies that robust value functions satisfies Bellman-type equations. Finally, \Cref{sec:sim} provides various simulation lemmas for robust values under different uncertainty sets.

\subsection{Dual Form of Robust Values}\label{sec:dual form}

We first present two classical lemmas for total variation distance and KL divergence, which can be found in \citet{panaganti2022robust,blanchet2023double}.

\begin{proposition}\label{prop: TV} 
Let \( \mathtt{D}_{f_1} \) be defined with the convex function \( f_1(t) = |t - 1|/2 \) corresponding to the total variation distance. Then,
\[
\inf_{ \mathtt{D}_{f_1}(P \| P_o) \leq \xi} \mathbb{E}_P [\ell(X)] = \max_{\lambda \in \mathbb{R}} \left\{\lambda - \Eb_{P_o}[(\lambda - \ell(X))_{+}] - \xi\max_x(\lambda - \ell(x))_+ \right\}.
\]
\end{proposition}

\begin{proposition}\label{prop: KL}
Let \( \mathtt{D}_{f_2} \) be defined with the convex function \( f_2(t) = t\log t \) corresponding to the KL-divergence. Then,
\[
\inf_{ \mathtt{D}_{f_2}(P \|  P_o) \leq \xi} \mathbb{E}_P [l(X)] =   \sup_{\lambda\geq 0 } \left\{ -\lambda\log\Eb_{P_o}\left[\exp(-\ell(X)/\lambda )\right] - \lambda \xi \right\}.
\]

\end{proposition}

Then, we show the dual form for the  newly proposed $\Pc$-type uncertainty sets.

\begin{proposition}\label{prop: TV P type}
If $B(\theta^*) = B_{\Pc}^1(\theta^*)$, then, we have
\begin{align*}
    &V_{B(\theta^*), R}^{\pi}  =  \max_{ \boldsymbol{\gamma}\geq 0, \boldsymbol{\lambda}  } \left\{ \sum_{x_H}(f(x_H) - \gamma_{x_H})\Pc^*(x_H) - \sum_{a_{1:H-1}} \max_{o_{1:H}} \left|\gamma_{x_H} - f(x_H) - \lambda_{(a_{H-1},x_{H-1})}\right| \xi \right\},
\end{align*}

where $f(x_H) = R(x_H)\pi(\tau_{H-1})$, and the maximum must be achieved when 
\[ -f(x_H)\leq \lambda_{(a_{H-1},x_{H-1})}\leq 0,\quad  0\leq \gamma_{x_H} \leq \max\left\{f(x_H) + \lambda_{a_{H-1},x_{H-1}}, 0\right\}\leq f(x_H).\]

\end{proposition}

\begin{proof}

We recall the definition of $V_{B_{\Pc}^1(\theta^*), R}^{\pi}$:
\[ V_{B_{\Pc}^1(\theta^*), R}^{\pi} = \min_{\theta\in B_{\Pc}^1(\theta^*)} \sum_{x_H}\Pc^{\theta}(x_H)\pi(\tau_{H-1}) R(x_H) \]

To simplify the notation, we denote $\pi(\tau_{H-1}) R(x_H)$ by $f(x_H)$. Then, we need to specify the equations that form $B_{\Pc}^1(\theta^*)$. 

The most important one is that for any $a_{1:H-1}$ $\sum_{o_2,\ldots,o_H}|\Pc(o_{2:H}|a_{1:H-1}) - \Pc^*(o_{2:H}|a_{1:H-1})|\leq \xi$. Another implicit constraint is that if we sum over all $o_H$, the value $\sum_{o_H}\Pc(o_H,\ldots,o_2|a_1,\ldots, a_{H-1})$ is constant with respect to $a_{H-1}$. Therefore, we construct some additional variables $\Pc(x_h)$ and $s(x_H) \geq |\Pc(x_H) - \Pc^*(x_H)|$ and formulate the following constrained optimization problem.

\begin{align*}
    \text{P1: } &\left\{
    \begin{aligned}
        &\min \sum_{x_H}\Pc(x_H)f(x_H)\\
        &\ \text{s.t.} \sum_{o_2}\Pc(x_2) = 1,\quad \forall (a_1,x_1) \\
        &\ \ \quad \sum_{o_3}\Pc(x_3) = \Pc(x_2),\quad \forall (a_2,x_2) \\
        &\ \ \quad \ldots \\
        &\ \ \quad \sum_{o_h}\Pc(x_h) = \Pc(x_{h-1}),\quad \forall (a_{h-1},x_{h-1})\\
        &\ \ \quad \ldots \\
        &\ \ \quad \sum_{o_H}\Pc(x_H) = \Pc(x_{H-1}),\quad \forall (a_{H-1},x_{H-1})\\
        &\ \ \quad -s(x_H) \leq \Pc(x_H) - \Pc^*(x_H) \leq s(x_H), \quad \forall x_H \\
        &\ \ \quad \sum_{o_{2:H}} s(x_H) \le \xi,\quad \forall a_{1:H-1} \\
        &\ \ \quad \Pc(x_H) \geq 0, s(x_H) \geq 0, \quad \forall x_H
    \end{aligned}
    \right.
\end{align*}

Since P1 is a linear programming, we can formulate the Lagrangian function $L_1 (\boldsymbol{\lambda}, \boldsymbol{\alpha},\boldsymbol{\beta}, \boldsymbol{\delta} ) $ as follows:
\begin{align*}
    L_1& (\boldsymbol{\lambda}, \boldsymbol{\alpha},\boldsymbol{\beta}, \boldsymbol{\gamma}, \boldsymbol{\delta}, \boldsymbol{\eta} )    \\
    &= \sum_{x_H}\Pc(x_H)f(x_H) + \sum_{(a_1,x_1)} \lambda_{(a_1,x_1)} \left( \sum_{o_2}\Pc(x_2) - 1 \right) + \sum_{h=2}^{H-1} \sum_{(a_h,x_h)} \lambda_{(a_h,x_h)} \left( \sum_{o_{h+1}}\Pc(x_{h+1}) - \Pc(x_{h}) \right) \\
    &\quad + \sum_{x_H}\alpha_{x_H}(\Pc(x_H) - \Pc^*(x_H) - s(x_H)) + \sum_{x_H} \beta_{x_H} ( - \Pc(x_H) + \Pc^*(x_H) - s(x_H))  \\
    &\quad - \sum_{x_H}( \gamma_{x_H} \Pc(x_H) + \delta_{x_H} s(x_H) ) + \sum_{a_{1:H-1}}\eta_{a_{1:H-1}} \left( \sum_{o_{2:H}} s(x_H) - \xi \right),
\end{align*}
where $\lambda_{(a_h,x_h)}\in \mathbb{R}$, and  $\alpha_{x_H} ,\beta_{x_H}, \gamma_{x_H}, \delta_{x_H}, \eta_{a_{1:H-1}} \geq 0,\quad \forall x_H$.

Thus, the dual problem is 
\begin{align*}
    \text{D1:} \left\{ 
    \begin{aligned}
        &\max_{\boldsymbol{\lambda}, \boldsymbol{\alpha},\boldsymbol{\beta}, \boldsymbol{\gamma}, \boldsymbol{\delta}, \eta} \inf_{\Pc, s} L_1(\boldsymbol{\lambda}, \boldsymbol{\alpha},\boldsymbol{\beta}, \boldsymbol{\gamma}, \boldsymbol{\delta}, \eta ) \\
        &\ \text{s.t.} \ f(x_H) + \lambda_{(a_{H-1}, x_{H-1})} + \alpha_{x_H} - \beta_{x_H} - \gamma_{x_H} = 0,\quad \forall x_H\\
        &\ \ \quad  \eta_{a_{1:H-1}} - \delta_{x_H} - \alpha_{x_H} - \beta_{x_H} = 0, \quad \forall x_H\\
        &\ \  \quad \lambda_{(a_{h-1}, x_{h-1})} = \sum_{a_h} \lambda_{(a_h,x_h)}, \quad \forall h, x_h \\
        & \ \ \quad \alpha_{x_H} ,\beta_{x_H}, \gamma_{x_H}, \delta_{x_H}, \eta_{a_{1:H-1}} \geq 0,\quad \forall x_H.
    \end{aligned}
    \right.
\end{align*}

Denote $\inf_{\Pc, s} L_1(\boldsymbol{\lambda}, \boldsymbol{\alpha},\boldsymbol{\beta}, \boldsymbol{\gamma}, \boldsymbol{\delta}, \eta )$ by $\bar{L}_1$. Then, we have
\begin{align*}
     \bar{L}_1  = \sum_{x_H}(\beta_{x_H} - \alpha_{x_H})\Pc^*(x_H) - \lambda_{(a_1,x_1)} - \sum_{a_{1:H-1}}\eta_{a_{1:H-1}} \xi. 
\end{align*}

In addition, by the third constraint in the dual problem D1, we have
\begin{align*}
    \lambda_{(a_1,x_1)} & = \sum_{a_2}\lambda_{(a_2,x_2)}\\
    & = \sum_{o_2}\Pc^*(x_2)\sum_{a_2}\lambda_{(a_2,x_2)}\\
    &  \ldots \\
    & = \sum_{x_H}\Pc^*(x_H)\lambda_{(a_{H-1}, x_{H-1})}.
\end{align*}

Thus, we can simplify $\bar{L}_1$ to
\begin{align*}
    \bar{L}_1(\boldsymbol{\lambda}, \boldsymbol{\alpha},\boldsymbol{\beta}, \boldsymbol{\gamma}, \boldsymbol{\delta}, \boldsymbol{\eta}) & = \sum_{x_H}(\beta_{x_H} - \alpha_{x_H})\Pc^*(x_H) - \lambda_{(a_1,x_1)} - \sum_{a_{1:H-1}}\eta_{a_{1:H-1}} \xi \\
    & = \sum_{x_H}(\beta_{x_H} - \alpha_{x_H} - \lambda_{(a_{H-1}, x_{H-1})} )\Pc^*(x_H)  - \sum_{a_{1:H-1}}\eta_{a_{1:H-1}} \xi \\
    & = \sum_{x_H}(f(x_H) - \gamma_{x_H})\Pc^*(x_H) - \sum_{a_{1:H-1}}\eta_{a_{1:H-1}} \xi. 
\end{align*}

Hence, by the Slater's condition, the primal value equals to the dual value. Mathematically, we have
\begin{align*}
    &V_{B(\theta^*), R}^{\pi} \\
    &\quad  = \max_{ \boldsymbol{\lambda}, \boldsymbol{\alpha},\boldsymbol{\beta}, \boldsymbol{\gamma}, \boldsymbol{\delta}, \boldsymbol{\eta} } \bar{L}_1 ( \boldsymbol{\lambda}, \boldsymbol{\alpha},\boldsymbol{\beta}, \boldsymbol{\gamma}, \boldsymbol{\delta}, \boldsymbol{\eta} ) \\
    & \quad =  \max_{ \boldsymbol{\lambda}, \boldsymbol{\alpha},\boldsymbol{\beta}, \boldsymbol{\gamma}, \boldsymbol{\delta}, \boldsymbol{\eta} } \left\{ \sum_{x_H}(f(x_H) - \gamma_{x_H})\Pc^*(x_H) - \sum_{a_{1:H-1}}\eta_{a_{1:H-1}} \xi 
 \right\} \\
    &\quad \quad \quad \text{s.t.} 
    \left\{
    \begin{aligned}
    & \eta_{a_{1:H-1}} = \alpha_{x_H} + \beta_{x_H} + \delta_{x_H}\\
    &\gamma_{x_H} = f(x_H) + \lambda_{(a_{H-1}, x_{H-1})} + \alpha_{x_H} - \beta_{x_H}\\
    &\eta_{a_{1:H-1}}, \gamma_{x_H}, \alpha_{x_H}, \beta_{x_H},\delta_{x_H} \geq 0.
    \end{aligned}
    \right.
\end{align*}

If we fix $\boldsymbol{\lambda}$ and $\boldsymbol{\gamma}$, then the maximum is achieved when
\[
\left\{
    \begin{aligned}
    & \eta_{a_{1:H-1}} = \max_{o_{1:H}} \left|\gamma_{x_H} - f(x_H) - \lambda_{(a_{H-1},x_{H-1})}\right|\\
    &\alpha_{x_H} = \max\left\{\gamma_{x_H} - f(x_H) - \lambda_{(a_{H-1},x_{H-1})}, 0 \right\}\\
    &\beta_{x_H}  = \max\left\{ - \gamma_{x_H} + f(x_H) + \lambda_{(a_{H-1},x_{H-1})}, 0 \right\}\\
    \end{aligned}
\right.
\]

Therefore, 
\begin{align*}
    &V_{B(\theta^*), R}^{\pi} \\
    & \quad =  \max_{ \boldsymbol{\gamma}\geq 0, \boldsymbol{\lambda}  } \left\{ \sum_{x_H}(f(x_H) - \gamma_{x_H})\Pc^*(x_H) - \sum_{a_{1:H-1}} \max_{o_{1:H}} \left|\gamma_{x_H} - f(x_H) - \lambda_{(a_{H-1},x_{H-1})}\right| \xi \right\} 
\end{align*}

Moreover, since $1\geq f(x_H)\geq 0$, the maximum must be achieved when 
\[ -f(x_H)\leq \lambda_{(a_{H-1},x_{H-1})}\leq 0,\quad  0\leq \gamma_{x_H} \leq \max\left\{f(x_H) + \lambda_{a_{H-1},x_{H-1}}, 0\right\}\leq f(x_H).\]
\end{proof}

\begin{proposition}\label{prop: KL P set}
    If $B(\theta^*) = B_{\Pc}^2(\theta^*)$, then, we have
    \begin{align*}
    &V_{B(\theta^*), R}^{\pi} = \max_{\boldsymbol{\eta}\geq 0 } \left\{ -\sum_{a_{H-1},x_{H-1}}\eta_{a_{1:H-1}}\Pc^{\theta^*}(x_{H-1})\log \Eb_{o\sim \Tc^{\theta^*}_{H-1}(\cdot|\tau_{H-1})}\left[ \exp\left(- \frac{f(x_{H}) }{ \eta_{a_{1:H-1}}} \right) \right] - \sum_{a_{1:H-1}}\eta_{a_{1:H-1}}\xi\right\}.
    \end{align*}
\end{proposition}

\begin{proof}

Similar to \Cref{prop: TV P type}, we denote $\pi(x_H) R(x_H)$ by $f(x_H)$. The robust value equals to the following optimization problem.

\begin{align*}
    \text{P2: } &\left\{
    \begin{aligned}
        &\min \sum_{x_H}\Pc(x_H)f(x_H)\\
        &\ \text{s.t.} \sum_{o_2}\Pc(x_2) = 1,\quad \forall (a_1,x_1) \\
        &\ \ \quad \sum_{o_3}\Pc(x_3) = \Pc(x_2),\quad \forall (a_2,x_2) \\
        &\ \ \quad \ldots \\
        &\ \ \quad \sum_{o_h}\Pc(x_h) = \Pc(x_{h-1}),\quad \forall (a_{h-1},x_{h-1})\\
        &\ \ \quad \ldots \\
        &\ \ \quad \sum_{o_H}\Pc(x_H) = \Pc(x_{H-1}),\quad \forall (a_{H-1},x_{H-1})\\
        &\ \ \quad  \sum_{o_{2:H}} \Pc(x_H)\log\frac{\Pc(x_H)}{\Pc^*(x_H)} \leq  \xi,\quad \forall a_{1:H-1} \\
        &\ \ \quad \Pc(x_H) \geq 0, \quad \forall x_H
    \end{aligned}
    \right.
\end{align*}

Since P2 is a convex optimization problem, we can formulate the Lagrangian function $L_2 (\boldsymbol{\lambda}, \boldsymbol{\gamma}, \boldsymbol{\eta} ) $ as follows
\begin{align*}
    L_2& (\boldsymbol{\lambda},  \boldsymbol{\gamma}, \boldsymbol{\eta} )    \\
    &= \sum_{x_H}\Pc(x_H)f(x_H) + \sum_{(a_1,x_1)} \lambda_{(a_1,x_1)} \left( \sum_{o_2}\Pc(x_2) - 1 \right) + \sum_{h=2}^{H-1} \sum_{(a_h,x_h)} \lambda_{(a_h,x_h)} \left( \sum_{o_{h+1}}\Pc(x_{h+1}) - \Pc(x_{h}) \right) \\
    &\quad - \sum_{x_H} \gamma_{x_H} \Pc(x_H)   + \sum_{a_{1:H-1}}\eta_{a_{1:H-1}} \left( \sum_{o_{2:H}} \Pc(x_H)\log\frac{\Pc(x_H)}{\Pc^*(x_H)} - \xi \right),
\end{align*}
where $\lambda_{(a_h,x_h)}\in \mathbb{R}$, and  $ \gamma_{x_H},  \eta_{a_{1:H-1}} \geq 0,\quad \forall x_H$.

Thus, the dual problem is 
\begin{align*}
    \text{D1:} \left\{ 
    \begin{aligned}
        &\max_{\boldsymbol{\lambda}, \boldsymbol{\gamma}, \boldsymbol{\eta} } \inf_{\Pc, s} L_2 (\boldsymbol{\lambda},  \boldsymbol{\gamma}, \boldsymbol{\eta} )  \\
        &\ \text{s.t.} \ f(x_H) + \lambda_{(a_{H-1}, x_{H-1})}    - \gamma_{x_H} +\eta_{a_{1:H-1}}\left(\log\frac{\Pc(x_H)}{\Pc^*(x_H)} + 1 \right) = 0,\quad \forall x_H\\
        &\ \  \quad \lambda_{(a_{h-1}, x_{h-1})} = \sum_{a_h} \lambda_{(a_h,x_h)}, \quad \forall h, x_h \\
        & \ \ \quad  \gamma_{x_H}, \eta_{a_{1:H-1}} \geq 0,\quad \forall x_H.
    \end{aligned}
    \right.
\end{align*}

Denote $\inf_{\Pc, s} L_2 (\boldsymbol{\lambda},  \boldsymbol{\gamma}, \boldsymbol{\eta} ) $ by $\bar{L}_2$. Then, we have
\begin{align*}
     \bar{L}_2  & = - \sum_{x_H} \eta_{a_{1:H-1}}\Pc(x_H) - \lambda_{(a_1,x_1)} - \sum_{a_{1:H-1}}\eta_{a_{1:H-1}} \xi \\
     & =  - \sum_{x_H} \eta_{a_{1:H-1}}\Pc(x_H) - \sum_{x_H}\Pc^*(x_H)\lambda_{(a_{H-1}, x_{H-1})} - \sum_{a_{1:H-1}}\eta_{a_{1:H-1}} \xi \\
     & = - \sum_{x_H} \eta_{a_{1:H-1}} \Pc^*(x_H)   \exp\left( \frac{ \gamma_{x_H} - f(x_H) - \lambda_{(a_{H-1},x_{H-1})} }{\eta_{a_{1:H-1}}} -1 \right)  - \sum_{x_H}\Pc^*(x_H)\lambda_{(a_{H-1}, x_{H-1})} - \sum_{a_{1:H-1}}\eta_{a_{1:H-1}} \xi.
\end{align*}

Hence, by the Slater's condition, the primal value equals to the dual value. Mathematically, we have
\begin{align*}
    &V_{B(\theta^*), R}^{\pi} \\
    &\quad  = \max_{ \boldsymbol{\lambda}, \boldsymbol{\gamma}, \boldsymbol{\eta} } \bar{L}_2 ( \boldsymbol{\lambda},  \boldsymbol{\gamma},  \boldsymbol{\eta} ) \\
    & \quad = \max_{\boldsymbol{\eta}\geq 0} \max_{ \boldsymbol{\gamma}\ge 0, \boldsymbol{\lambda} } \left\{ - \sum_{x_H} \eta_{a_{1:H-1}} \Pc^*(x_H)   \exp\left( \frac{ \gamma_{x_H} - f(x_H)  - \lambda_{(a_{H-1},x_{H-1})} }{\eta_{a_{1:H-1}}} -1 \right)  \right.\\
    &\hspace{2cm} \qquad \qquad \left. - \sum_{x_H}\Pc^*(x_H)\lambda_{(a_{H-1}, x_{H-1})} - \sum_{a_{1:H-1}}\eta_{a_{1:H-1}} \xi \right\} \\
    &\quad = \max_{\boldsymbol{\eta}\geq 0 } \left\{ -\sum_{a_{H-1},x_{H-1}}\eta_{a_{H-1},x_{H-1}}\Pc^*(x_{H-1})\log \Eb_{o\sim \Tc^*_{H-1}(\cdot|\tau_{H-1})}\left[ \exp\left(- \frac{f(x_{H}) }{ \eta_{a_{1:H-1}}} \right) \right] - \sum_{a_{1:H-1}}\eta_{a_{1:H-1}}\xi\right\},
\end{align*}

where we choose $\gamma_{x_H} = 0$ and 
\[ \lambda_{(a_{H-1},x_{H-1})} = \eta_{a_{1:H-1}} \log \left( \sum_{o}\Tc^*_{H-1}(o|\tau_{H-1} ) \exp\left(\frac{-f(x_H)}{\eta_{a_{1:H-1}}} - 1 \right) \right). \]

\end{proof}

\subsection{Bellman-type Equations under $\Tc$-type Uncertainty Sets}\label{sec:bellman type}

Recall that robust value function is defined as follows.
\begin{align*}
    &V_{B(\theta^*),R,h}^{\pi}(\tau_{h-1},o_h) = \inf_{\theta\in B(\theta^*)}\Eb_{\theta}^{\pi} \left[   \sum_{h' = h}^H R(\btau_{h'}) \bigg| \tau_{h-1}, o_h\right],\\
    &Q_{B(\theta^*),R,h}^{\pi}(\tau_{h}) = \inf_{\theta\in B(\theta^*)}\Eb_{\theta}^{\pi} \left[  \sum_{h' = h}^H R(\btau_{h'}) \bigg| \tau_{h}\right].
\end{align*}

Then, we have the following Bellman-type equations for robust value functions.
\begin{lemma}
    For any $\theta$, we have
\begin{align*}
    &\left\{
    \begin{aligned}
        &V_{B(\theta^*),R,h}^{\pi}(\tau_{h-1},o_h) = \sum_{a_h}\pi(a_h|\tau_{h-1},o_h)Q_{B(\theta^*),R,h}^{\pi}(\tau_h)\\
        &Q_{B(\theta^*),R,h}^{\pi}(\tau_h) = R(\tau_h) + \inf_{\theta\in B(\theta^*)} \sum_{o_{h+1}}\Tc_h^{\theta}(o_{h+1}|\tau_h) V_{B(\theta^*), R, h+1}^{\pi}(\tau_h, o_{h+1})
    \end{aligned}
    \right..
\end{align*}

\end{lemma}

\begin{proof}
    The first equation is straightforward.

    For the second equation, by the definition of robust Q-value function, for any $\iota>0$, there exists $\theta_{\iota} \in B(\theta^*)$ such that 
    \[Q_{B(\theta^*), R, h}^{\pi} (\tau_h) \geq \Eb_{\theta_{\iota}}^{\pi} \left[\sum_{h'=h}^H R(\btau_{h'}) \bigg| \tau_h \right] - \iota. \]

    Then, we have
    \begin{align*}
        Q_{B(\theta^*), R, h}^{\pi} (\tau_h) &\geq \Eb_{\theta_{\iota}}^{\pi} \left[\sum_{h'=h}^H R(\btau_{h'}) \bigg| \tau_h \right] - \iota \\
        & = R(\tau_h) + \sum_{o_{h+1}} \Tc_h^{\theta_{\iota}}(o_{h+1}|\tau_h) V_{\theta_{\iota}, R, h+1}^{\pi} (\tau_h, o_{h+1}) - \iota\\
        & \geq R(\tau_h) + \inf_{\theta\in B(\theta^*)}\sum_{o_{h+1}} \Tc_h^{\theta}(o_{h+1}|\tau_h) V_{ B(\theta^*), R, h+1}^{\pi} (\tau_h, o_{h+1}) - \iota.
    \end{align*}

    Let $\iota\rightarrow 0$, we conclude that 
    \[ Q_{B(\theta^*), R, h}^{\pi} (\tau_h)  \geq R(\tau_h) + \inf_{\theta\in B(\theta^*)}\sum_{o_{h+1}} \Tc_h^{\theta}(o_{h+1}|\tau_h) V_{ B(\theta^*), R, h+1}^{\pi} (\tau_h, o_{h+1}). \] 

    On the other hand, for any $\iota_1, \iota_2>0$, there exist $\theta_{\iota_1}, \theta_{\iota_2} \in B(\theta^*)$ such that 
    \[ V_{B(\theta^*), R, h+1}^{\pi} (\tau_h, o_{h+1}) \geq V_{\theta_{\iota_1}, R, h+1}^{\pi} (\tau_h, o_{h+1}) - \iota_1, \]
    and 
    \[  \inf_{\theta\in B(\theta^*)}\sum_{o_{h+1}} \Tc_h^{\theta}(o_{h+1}|\tau_h)V_{B(\theta^*), R, h+1}^{\pi}(\tau_h, o_{h+1}) \geq  \sum_{o_{h+1}} \Tc_h^{\theta_{\iota_2}} (o_{h+1}|\tau_h) V_{B(\theta^*), R, h+1}^{\pi}(\tau_h, o_{h+1}) - \iota_2. \]

    In addition, there exists $\theta_3\in B(\theta^*)$ such that $\Tc_{\theta_3}(o_{h+1}|\tau_h) = \Tc_{\theta_{\iota_2}}(o_{h+1}|\tau_h)$ and $V_{\theta_{\iota_1}, R, h+1}^{\pi} = V_{\theta_3, R, h+1}^{\pi}$. Then, we have
    \begin{align*}
        R(\tau_h) & + \inf_{\theta\in B(\theta^*)}\sum_{o_{h+1}} \Tc_h^{\theta}(o_{h+1}|\tau_h)V_{B(\theta^*), R, h+1}^{\pi}(\tau_h, o_{h+1}) \\
        &\geq R(\tau_h) + \sum_{o_{h+1}} \Tc_h^{\theta_3}(o_{h+1}|\tau_h)V_{B(\theta^*), R, h+1}^{\pi}(\tau_h, o_{h+1}) - \iota_2\\
        &\geq R(\tau_h) + \sum_{o_{h+1}}\Tc_h^{\theta_3}(o_{h+1}|\tau_h)V_{\theta_3, R, h+1}^{\pi}(\tau_h, o_{h+1}) - \iota_1 - \iota_2\\
        & = \Eb_{\theta_3}^{\pi} \left[ \sum_{h'=h}^H R(\btau_h) \bigg|\tau_h \right] - \iota_1 - \iota_2\\
        &\geq Q_{B(\theta^*), R, h}^{\pi}(\tau_h) - \iota_1 - \iota_2.
    \end{align*}

    Let $\iota_1,\iota_2\rightarrow 0$, we can conclude that
    \[ R(\tau_h) + \inf_{\theta \in B(\theta^*)} \sum_{o_{h+1}} \Tc_h^{\theta} (o_{h+1}|\tau_h) V_{B(\theta^*), R, h+1}^{\pi} (\tau_h, o_{h+1}) \geq Q_{B(\theta^*), R, h}^{\pi} (\tau_h), \]
    which finishes the proof.
    
\end{proof}

\subsection{Simulation Lemmas of Robust Values}\label{sec:sim}

\begin{lemma}[Simulation lemma of robust value under $B_{\Pc}^1$]\label{lemma:Sim P TV}
    \[ \left|V_{B_{\Pc}^1(\theta^*), R}^{\pi} - V_{B_{\Pc}^1(\theta), R}^{\pi} \right| \leq  \left\|  \Db_{\theta}^{\pi} - \Db_{\theta^*}^{\pi}  \right\|_1\]
\end{lemma}

\begin{proof}

Let $f(x_H) = \pi(\tau_{H-1})R(x_H)$.

We first prove that $V_{B_{\Pc}^1(\theta^*), R}^{\pi} - V_{B_{\Pc}^1(\theta), R}^{\pi}   \leq  \left\|  \Db_{\theta}^{\pi} - \Db_{\theta^*}^{\pi}  \right\|_1$. By \Cref{prop: TV P type}, we have

\begin{align*}
    &V_{B(\theta^*), R}^{\pi} - V_{B(\theta), R}^{\pi} \\
    &\quad  =  \max_{ \boldsymbol{\gamma}\geq 0, \boldsymbol{\lambda}  } \left\{ \sum_{x_H}(f(x_H) - \gamma_{x_H})\Pc^{\theta^*}(x_H) - \sum_{a_{1:H-1}} \max_{o_{1:H}} \left|\gamma_{x_H} - f(x_H) - \lambda_{(a_{H-1},x_{H-1})}\right| \xi \right\} \\
    &\quad \quad -  \max_{ \boldsymbol{\gamma}\geq 0, \boldsymbol{\lambda}  } \left\{ \sum_{x_H}(f(x_H) - \gamma_{x_H})\Pc^{\theta}(x_H) - \sum_{a_{1:H-1}} \max_{o_{1:H}} \left|\gamma_{x_H} - f(x_H) - \lambda_{(a_{H-1},x_{H-1})}\right| \xi \right\},
\end{align*}

Let $ \boldsymbol{\gamma}^*,\boldsymbol{\lambda}^*$ be the solution to \[  \max_{ \boldsymbol{\gamma}\geq 0, \boldsymbol{\lambda}  } \left\{ \sum_{x_H}(f(x_H) - \gamma_{x_H})\Pc^{\theta^*}(x_H) - \sum_{a_{1:H-1}} \max_{o_{1:H}} \left|\gamma_{x_H} - f(x_H) - \lambda_{(a_{H-1},x_{H-1})}\right| \xi \right\}. \]

Again, from \Cref{prop: TV P type}, we know that 
\begin{equation}
-f(x_H)\leq \lambda^*_{(a_{H-1},x_{H-1})}\leq 0,\quad  0\leq \gamma^*_{x_H} \leq \max\left\{f(x_H) + \lambda^*_{a_{H-1},x_{H-1}}, 0\right\}\leq f(x_H).\label{eqn: range of lambda}
\end{equation}

Thus, the difference can be upper bounded as follows.
\begin{align*}
    &V_{B(\theta^*), R}^{\pi} - V_{B(\theta), R}^{\pi} \\
    &\quad \leq  \sum_{x_H}(f(x_H) - \gamma^*_{x_H})\Pc^{\theta^*}(x_H) - \sum_{a_{1:H-1}} \max_{o_{1:H}} \left|\gamma^*_{x_H} - f(x_H) - \lambda^*_{(a_{H-1},x_{H-1})}\right| \xi   \\
    &\quad \quad - \sum_{x_H}(f(x_H) - \gamma^*_{x_H})\Pc^{\theta}(x_H) + \sum_{a_{1:H-1}} \max_{o_{1:H}} \left|\gamma^*_{x_H} - f(x_H) - \lambda^*_{(a_{H-1},x_{H-1})}\right| \xi \\
    &\quad = \sum_{x_H}(f(x_H) - \gamma^*_{x_H})(\Pc^{\theta^*}(x_H) -\Pc^{\theta}(x_H) ) \\
    &\quad \leq \sum_{x_H} |f(x_H) - \gamma^*_{x_H}||\Pc^{\theta^*}(x_H) -\Pc^{\theta}(x_H) |\\
    &\quad \overset{(a)}\leq \sum_{x_H} f(x_H)|\Pc^{\theta^*}(x_H) -\Pc^{\theta}(x_H) |\\
    &\quad = \sum_{x_H} \pi(\tau_{H-1})R(x_H)|\Pc^{\theta^*}(x_H) -\Pc^{\theta}(x_H) | \\
    &\quad \leq \left\| \Db_{\theta}^{\pi} - \Db_{\theta^*}^{\pi} \right\|_1,
\end{align*}
where $(a)$ is due to \Cref{eqn: range of lambda}.

The opposite side $ V_{B_{\Pc}^1(\theta^*), R}^{\pi} - V_{B_{\Pc}^1(\theta), R}^{\pi}   \geq  -\left\|  \Db_{\theta}^{\pi} - \Db_{\theta^*}^{\pi}  \right\|_1$ follows exactly the same proof.

\end{proof}

\begin{lemma}[Simulation lemma of robust value under $B_{\Pc}^2$]\label{lemma:Sim P KL}
    \[ \left|V_{B_{\Pc}^2(\theta^*), R}^{\pi} - V_{B_{\Pc}^2(\theta), R}^{\pi} \right| \leq 3\exp(1/\eta^*) \left\|  \Db_{\theta}^{\pi} - \Db_{\theta^*}^{\pi}  \right\|_1,\]
    where $\eta^*$ is a lower bound of the optimal solution 
     \[\eta^*\leq\min 
     \left\{
     \begin{aligned}
     &\arg\max_{\boldsymbol{\eta}\geq 0 } \left\{ -\sum_{a_{H-1},x_{H-1}}\eta_{a_{H-1},x_{H-1}}\Pc^{\theta^*}(x_{H-1})\log \Eb_{o\sim \Tc^{\theta^*}_{H-1}(\cdot|\tau_{H-1})}\left[ \exp\left(- \frac{f(x_{H}) }{ \eta_{a_{1:H-1}}} \right) \right] \right.\\
     &\quad \hspace{2cm} \left.- \sum_{a_{1:H-1}}\eta_{a_{1:H-1}}\xi\right\},\\
     &\arg\max_{\boldsymbol{\eta}\geq 0 } \left\{ -\sum_{a_{H-1},x_{H-1}}\eta_{a_{H-1},x_{H-1}}\Pc^{\theta}(x_{H-1})\log \Eb_{o\sim \Tc^{\theta}_{H-1}(\cdot|\tau_{H-1})}\left[ \exp\left(- \frac{f(x_{H}) }{ \eta_{a_{1:H-1}}} \right) \right] \right.\\
     &\quad \hspace{2cm} \left. - \sum_{a_{1:H-1}}\eta_{a_{1:H-1}}\xi\right\}.
     \end{aligned}
     \right.\]
     Here, if $x\leq \mathbf{v}$, where $x$ is a scalar and $\mathbf{v}$ is a vector, then $x$ is less than or equal to each coordinate of the vector.

\end{lemma}

\begin{proof}

Let $f(x_H) = \pi(\tau_{H-1})R(x_H)$.
By \Cref{prop: KL P set}, we have
\begin{align*}
    &V_{B(\theta^*), R}^{\pi} - V_{B(\theta), R}^{\pi} \\
    &\quad  = \max_{\boldsymbol{\eta}\geq 0 } \left\{ -\sum_{a_{H-1},x_{H-1}}\eta_{a_{H-1},x_{H-1}}\Pc^{\theta^*}(x_{H-1})\log \Eb_{o\sim \Tc^{\theta^*}_{H-1}(\cdot|\tau_{H-1})}\left[ \exp\left(- \frac{f(x_{H}) }{ \eta_{a_{1:H-1}}} \right) \right] - \sum_{a_{1:H-1}}\eta_{a_{1:H-1}}\xi\right\} \\
    &\quad \quad - \max_{\boldsymbol{\eta}\geq 0 } \left\{ -\sum_{a_{H-1},x_{H-1}}\eta_{a_{H-1},x_{H-1}}\Pc^{\theta'}(x_{H-1})\log \Eb_{o\sim \Tc^{\hat{\theta}}_{H-1}(\cdot|\tau_{H-1})}\left[ \exp\left(- \frac{f(x_{H}) }{ \eta_{a_{1:H-1}}} \right) \right] - \sum_{a_{1:H-1}}\eta_{a_{1:H-1}}\xi \right\} .
\end{align*}

Let $ \boldsymbol{\eta}^*$ be the solution to \[\max_{\boldsymbol{\eta}\geq 0 } \left\{ -\sum_{a_{H-1},x_{H-1}}\eta_{a_{H-1},x_{H-1}}\Pc^{\theta^*}(x_{H-1})\log \Eb_{o\sim \Tc^{\theta^*}_{H-1}(\cdot|\tau_{H-1})}\left[ \exp\left(- \frac{f(x_{H}) }{ \eta_{a_{1:H-1}}} \right) \right] - \sum_{a_{1:H-1}}\eta_{a_{1:H-1}}\xi\right\}. \]

Then, we difference can be upper bounded as follows.
\begin{align*}
    &V_{B(\theta^*), R}^{\pi} - V_{B(\theta), R}^{\pi} \\
    &\quad \leq  \sum_{a_{H-1},x_{H-1}}\eta^*_{a_{H-1},x_{H-1}} \Pc^{ \theta } \log \Eb_{o_H\sim \Tc^{ \theta }_{H-1}(\cdot|\tau_{H-1})}\left[ \exp\left(- \frac{f(x_{H}) }{ \eta^*_{a_{1:H-1}} } \right) \right] + \sum_{a_{1:H-1}}\eta_{a_{1:H-1}}^* \xi   \\
    &\quad \quad - \sum_{a_{H-1},x_{H-1}}\eta^*_{a_{H-1},x_{H-1}} \Pc^{\theta^*}(x_{H-1})\log \Eb_{o_H\sim \Tc^{\theta^*}_{H-1}(\cdot|\tau_{H-1})}\left[ \exp\left(- \frac{f(x_{H}) }{ \eta^*_{a_{1:H-1}} } \right) \right] - \sum_{a_{1:H-1}}\eta^*_{a_{1:H-1}} \xi \\
    &\quad = \sum_{a_{H-1},x_{H-1}}\eta^*_{a_{H-1},x_{H-1}} \left( \Pc^{ \theta }(x_{H-1}) - \Pc^{\theta^*}(x_{H-1}) \right) \log \Eb_{o_H\sim \Tc^{\theta}_{H-1}(\cdot|\tau_{H-1})}\left[ \exp\left(- \frac{f(x_{H}) }{ \eta^*_{a_{1:H-1}} } \right) \right] \\
    &\quad \quad + \sum_{a_{H-1},x_{H-1}}\eta^*_{a_{H-1},x_{H-1}} \Pc^{\theta^*}(x_{H-1}) \left( \log \Eb_{o_H\sim \Tc^{\theta}_{H-1}(\cdot|\tau_{H-1})}\left[ \exp\left(- \frac{f(x_{H}) }{ \eta^*_{a_{1:H-1}} } \right) \right] \right. \\
    &\quad\quad \left.- \log \Eb_{o_H\sim \Tc^{\theta^*}_{H-1}(\cdot|\tau_{H-1})}\left[ \exp\left(- \frac{f(x_{H}) }{ \eta^*_{a_{1:H-1}} } \right) \right] \right) \\
    &\quad \leq \sum_{a_{H-1},x_{H-1}} -\eta^*_{a_{H-1},x_{H-1}} \left| \Pc^{\theta}(x_{H-1}) - \Pc^{\theta^*}(x_{H-1}) \right| \log \Eb_{o_H\sim \Tc^{\theta}_{H-1}(\cdot|\tau_{H-1})} \left[ \exp\left(- \frac{f(x_{H}) }{ \eta^*_{a_{1:H-1}} } \right) \right] \\
    &\quad \quad + \sum_{a_{H-1},x_{H-1}}\eta^*_{a_{H-1},x_{H-1}} \Pc^{\theta^*}(x_{H-1}) \log \frac{ \Eb_{o_H\sim \Tc^{\theta}_{H-1}(\cdot|\tau_{H-1})}\left[ \exp\left(- \frac{f(x_{H}) }{ \eta^*_{a_{1:H-1}} } \right) \right] }{  \Eb_{o_H \sim \Tc^{\theta^*}_{H-1}(\cdot|\tau_{H-1})}\left[ \exp\left(- \frac{f(x_{H}) }{ \eta^*_{a_{1:H-1}} } \right) \right] } \\
    &\quad \leq \sum_{a_{H-1},x_{H-1}}  \left| \Pc^{\theta}(x_{H-1}) - \Pc^{\theta^*}(x_{H-1}) \right|  \max_{o_H} f(x_{H})  \\
    &\quad \quad + \sum_{a_{H-1},x_{H-1}}\eta^*_{a_{H-1},x_{H-1}} \Pc^{\theta^*}(x_{H-1})  \underbrace{\log \frac{ \Eb_{o_H\sim \Tc^{\theta}_{H-1}(\cdot|\tau_{H-1})}\left[ \exp\left(- \frac{f(x_{H}) }{ \eta^*_{a_{1:H-1}} } \right) \right] }{  \Eb_{o_H \sim \Tc^{\theta^*}_{H-1}(\cdot|\tau_{H-1})}\left[ \exp\left(- \frac{f(x_{H}) }{ \eta^*_{a_{1:H-1}} } \right) \right] } }_{I_0}
\end{align*}

For term $I_0$, we fix $\tau_{H-1}$ and write $f(x_H)$ as $\pi(\tau_{H-1})R(o_H)$ for convenience.

Since $\tau_{H-1}$ is fixed and  we aim to upper bound the term 
\[I_0 = \log \frac{ \Eb_{o_H\sim \Tc^{\hat{\theta}}_{H-1}(\cdot|\tau_{H-1})}\left[ \exp\left(- \frac{ \pi(\tau_{H-1}) R(o_H) }{ \eta^*_{a_{1:H-1}} } \right) \right] }{  \Eb_{o_H \sim \Tc^{\theta^*}_{H-1}(\cdot|\tau_{H-1})}\left[ \exp\left(- \frac{ \pi(\tau_{H-1}) R(o_H) }{ \eta^*_{a_{1:H-1}} } \right) \right] }, \]

$\pi(\tau_{H-1})/\eta^*_{a_{1:H-1}} \in [0, 1/\eta^*]$ can be treated as a variable. We define a function $g(t)$ for $1/\eta^*\geq t\geq 0$ as follows

\[ g(t) = \log \frac{ \Eb_{o\sim P }\left[ \exp\left(-  t R(o) \right) \right] }{  \Eb_{o \sim Q }\left[ \exp\left(-  t R(o) \right) \right] },  \]

where $P, Q \in \Delta(\mathcal{O})$ are two distributions over the observation space.

Taking the derivative with respect to $t$, we have
\begin{align*}
    \frac{d g}{ d t} & = \frac{ \Eb_{o\sim P }\left[ - R(o) \exp\left(-  t R(o) \right) \right] }{ \Eb_{o\sim P }\left[ \exp\left(-  t R(o) \right) \right] }  - \frac{  \Eb_{o \sim Q }\left[ -R(o) \exp\left(-  t R(o) \right) \right] }{  \Eb_{o \sim Q }\left[ \exp\left(-  t R(o) \right) \right] } \\
    & = \Eb_{o \sim Q }\left[ R(o) \frac{ \exp\left(-  t R(o)  \right) }{ \Eb_{ o \sim Q } \left[ \exp\left(-  t R(o)  \right) \right]  }\right] - \Eb_{o \sim P }\left[ R(o) \frac{ \exp\left(-  t R(o)  \right) }{ \Eb_{ o \sim P } \left[ \exp\left(-  t R(o)  \right) \right]  }\right].
\end{align*}

Due to $R(o)\in[0,1]$ and $1/\eta^* \geq t\geq 0$, we have 
\[ \left| \frac{dg}{dt} \right| \leq  \|P - Q\|_1 \frac{1 }{ \exp(- 1/\eta^*) }.\]

Thus, for $t\in[0,1/\eta^*]$, we have
\[g(t) = \int_0^t g'(t) dt \leq t \|P - Q\|_1 \exp(1/\eta^*).  \]

Plugging the upper bound of $g(t)$ back to the robust value difference, we have
\begin{align*}
    &V_{B(\theta^*), R}^{\pi} - V_{B(\theta), R}^{\pi} \\
    &\quad \leq \sum_{a_{H-1},x_{H-1}}  \left| \Pc^{\theta}(x_{H-1}) - \Pc^{\theta^*}(x_{H-1}) \right|  \max_{o_H} f(x_{H})  \\
    &\quad \quad + \sum_{a_{H-1},x_{H-1}}\eta^*_{a_{H-1},x_{H-1}} \Pc^{\theta^*}(x_{H-1})  \left\| \Tc^{\theta}_{H-1}(\cdot|\tau_{H-1}) - \Tc^{\theta^*}_{H-1}(\cdot|\tau_{H-1}) \right\|_1 \frac{ \pi(\tau_{H-1}) }{ \eta_{a_{1:H-1} } } \exp(1/\eta^*)  \\
    &\leq \sum_{\tau_{H-1}} \left| \Db_{\theta}^{\pi}(\tau_{H-1}) - \Db_{\theta^*}^{\pi}(\tau_{H-1}) \right| + \exp(1/\eta^*) \Eb_{\theta^*}^{\pi}\left[  \left\| \Tc^{\theta}_{H-1}(\cdot|\tau_{H-1}) - \Tc^{\theta^*}_{H-1}(\cdot|\tau_{H-1}) \right\|_1  \right] \\
    &\overset{(a)}\leq 3\exp(1/\eta^*) \left\|  \Db_{\theta}^{\pi} - \Db_{\theta^*}^{\pi}  \right\|_1,
\end{align*}
where $(a)$ is due to \Cref{lemma: E TV < TV}.

For the second inequality, we follow the same argument and the proof is complete.

\end{proof}

\begin{lemma}[Simulation lemma for robust value under $B_{\Tc}^i$]\label{lemma:Sim T}
    
\begin{align*}
    &\left|V_{B_{\Tc}^1(\theta^*), R}^{\pi} - V_{B_{\Tc}^1(\theta), R}^{\pi} \right| \leq 2 C_B \| 
    \Db_{\theta^*}^{\pi} - \Db_{\theta}^{\pi} \|_1,\\
    &\left|V_{B_{\Tc}^2(\theta^*), R}^{\pi} - V_{B_{\Tc}^2(\theta), R}^{\pi} \right|  \leq 2 C_B C_{\Tc}^{\KL} \| 
    \Db_{\theta^*}^{\pi} - \Db_{\theta}^{\pi} \|_1\\
\end{align*}
where 
\( C_{\Tc}^{\KL} = \max\left\{\exp(\xi)/\xi, \lambda^*\exp(1/\lambda^*) \right\} 
\), and $\lambda^*$ is a lower bound for the optimal solutions of the following problems

\[\lambda^*= \min_{h,\tau_h}
\left\{
\begin{aligned}
&\arg \sup_{\lambda\geq 0} \left\{ -\lambda \log \Eb_{\Tc_h^{\theta^*}}[\exp(-V_{B_{\Tc}^2(\theta), R, h+1}^{\pi}(\tau_{h},o)/\lambda ) ] - \lambda \xi \right\}, \\
&\arg \sup_{\lambda\geq 0} \left\{ -\lambda \log \Eb_{\Tc_h^{\theta}}[\exp(-V_{B_{\Tc}^2(\theta^*), R, h+1}^{\pi}(\tau_{h},o)/\lambda ) ] - \lambda \xi \right\}.
\end{aligned}
\right.
\]

\end{lemma}

\begin{proof}
First, we have
    \begin{align*}
    V_{B_{\Tc}^i(\theta^*), R}^{\pi} &- V_{B_{\Tc}^i(\theta), R}^{\pi} \\
    & = \Eb^{\pi} \left[ Q_{B_{\Tc}^1(\theta^*), R, 1}^{\pi}(\tau_1) - Q_{B_{\Tc}^i(\theta) , R, 1}^{\pi}(\tau_1) \bigg|o_1 \right] \\
    & = \Eb^{\pi} \left[ \inf_{\theta'\in B_{\Tc}^i(\theta^*)} \sum_{o_{2}} \Tc_1^{\theta'}(o_2|\tau_1) V_{B_{\Tc}^i(\theta^*), R, 2}^{\pi} (\tau_1, o_2) - \inf_{\theta'\in B_{\Tc}^i(\theta)} \sum_{o_{2}} \Tc_1^{\theta'} (o_2|\tau_1) V_{B_{\Tc}^i(\theta), R, 2}^{\pi} (\tau_1, o_2)   \right] \\
    & = \Eb^{\pi} \left[ \inf_{\theta'\in B_{\Tc}^i(\theta^*)} \sum_{o_{2}}\Tc_1^{\theta} (o_2|\tau_1) V_{B_{\Tc}^i(\theta^*), R, 2}^{\pi} (\tau_1, o_2) - \inf_{\theta'\in B_{\Tc}^i(\theta^*)} \sum_{o_{2}} \Tc_1^{\theta}(o_2|\tau_1) V_{B_{\Tc}^i(\theta), R, 2}^{\pi} (\tau_1, o_2)  \right.\\
    &\quad\quad\quad \left. + \inf_{\theta'\in B_{\Tc}^i(\theta^*)} \sum_{o_{2}} \Tc_1^{\theta'}(o_2|\tau_1) V_{B_{\Tc}^i(\theta), R, 2}^{\pi} (\tau_1, o_2)  - \inf_{\theta'\in B_{\Tc}^i(\theta)} \sum_{o_{2}} \Tc_1^{\theta'}(o_2|\tau_1)V_{B_{\Tc}^i(\theta), R, 2}^{\pi} (\tau_1, o_2)   \right] \\
    &\leq  \max_{\theta'\in B_{\Tc}^1(\theta^*)}\Eb_{\theta'}^{\pi}\left[ V_{B_{\Tc}^i(\theta^*), R, 2}^{\pi} (\tau_1, o_2) - V_{B_{\Tc}^i(\theta), R, 2}^{\pi} (\tau_1, o_2) \right] + \Eb^{\pi}\left[ \mathcal{E}_{B_{\Tc}^i}(\theta, \tau_1) \right] \\
    &\leq \ldots\\
    &\leq \max_{\theta'\in B_{\Tc}^i(\theta^*)}\sum_{h=1}^H \Eb_{\theta'}^{\pi}\left[ \mathcal{E}_{B_{\Tc}^i}(\theta, \tau_h) \right] \leq C_B \sum_{h=1}^H \Eb_{\theta^*}^{\pi}\left[ \mathcal{E}_{B_{\Tc}^i}(\theta, \tau_h) \right]
\end{align*}

When $i = 1$, we have
\begin{align*}
    \mathcal{E}_{B_{\Tc}^1}(\theta, \tau_h) & =     \min_{\theta'\in B_{\Tc}^1(\theta^*)} \mathop{\Eb}_{o\sim\Tc_h^{\theta'}(\cdot|\tau_h)}\left[ V_{B_{\Tc}^1(\theta), R, h+1}^{\pi}(\tau_{h},o) \right]  - \min_{\theta'\in B_{\Tc}^1(\theta)} \mathop{\Eb}_{o\sim\Tc_h^{\theta'}(\cdot|\tau_h)} \left[V_{B_{\Tc}^1(\theta), R, h+1}^{\pi}(\tau_{h}, o) \right]\\
    & \overset{(a)} =    \sup_{\lambda\in[0,1]} \left\{ -\Eb_{ \Tc_h^{ \theta^* } }[ (\lambda - V_{B_{\Tc}^1(\theta), R, h+1}^{\pi}(\tau_{h},o) )_+ ] - \xi \max_{o}(\lambda - V_{B_{\Tc}^1(\theta), R, h+1}^{\pi}(\tau_{h},o) )_+ + \lambda \right\}  \\
     &\qquad    - \sup_{\lambda\in[0,1]} \left\{ -\Eb_{\Tc_h^{\theta}}[ (\lambda - V_{B_{\Tc}^1(\theta), R, h+1}^{\pi}(\tau_{h},o) )_+ ] - \xi \max_{o}(\lambda -V_{B_{\Tc}^1(\theta), R, h+1}^{\pi}(\tau_{h},o) )_+ + \lambda \right\}   \\
     &\leq \sup_{\lambda\in[0,1]} \left\{ \Eb_{\Tc_h^{\theta}}[ (\lambda - V_{B_{\Tc}^1(\theta), R, h+1}^{\pi}(\tau_{h},o) )_+ ] - \Eb_{\Tc_h^{ \theta^* }}[ (\lambda - V_{B_{\Tc}^1(\theta), R, h+1}^{\pi}(\tau_{h},o) )_+ ] \right\} \\
     &\leq \left\| \Tc_h^{ \theta }(\cdot| \tau_h) - \Tc_h^{\theta^*}(\cdot|\tau_h)\right\|_1,
\end{align*}
where $(a)$ is due to \Cref{prop: TV}.

When $i = 2$, we have
\begin{align*}
    \mathcal{E}_{B_{\Tc}^2}(\theta, \tau_h) & = \min_{\theta'\in B_{\Tc}^2(\theta^*)} \mathop{\Eb}_{o\sim\Tc_h^{\theta'}(\cdot|\tau_h)}\left[ V_{B_{\Tc}^2(\theta), R, h+1}^{\pi}(\tau_{h},o) \right]  - \min_{\theta'\in B_{\Tc}^2(\theta)} \mathop{\Eb}_{o\sim\Tc_h^{\theta'}(\cdot|\tau_h)} \left[V_{B_{\Tc}^2(\theta), R, h+1}^{\pi}(\tau_{h}, o) \right] \\
    & \overset{(a)}= \sup_{\lambda\geq 0} \left\{ -\lambda \log \Eb_{\Tc_h^{\theta^*}}[\exp(-V_{B_{\Tc}^2(\theta), R, h+1}^{\pi}(\tau_{h},o)/\lambda ) ] - \lambda \xi \right\} \\
     &\qquad  - \sup_{\lambda \geq 0 } \left\{ - \lambda \log \Eb_{\Tc_h^{\theta}}[ \exp(-V_{B_{\Tc}^2(\theta), R, h+1}^{\pi}(\tau_{h},o)/\lambda ) ] -  \lambda \xi \right\},
\end{align*}
where $(a)$ is due to \Cref{prop: KL}.

Let $\lambda^*(\tau_h) = \arg \sup_{\lambda\geq 0} \left\{ -\lambda \log \Eb_{\Tc_h^{\theta^*}}[\exp(-V_{B_{\Tc}^2(\theta), R, h+1}^{\pi}(\tau_{h},o)/\lambda ) ] - \lambda \xi \right\}$.

Then, we further have
\begin{align*}
     \mathcal{E}_{B_{\Tc}^2}(\theta, \tau_h) & \leq \lambda^*(\tau_h) \log \Eb_{\Tc_h^{\theta}}[ \exp(-V_{B_{\Tc}^2(\theta), R, h+1}^{\pi}(\tau_{h},o)/\lambda^*(\tau_h) ) ]\\
     &\qquad - \lambda^*(\tau_h) \log \Eb_{\Tc_h^{\theta^*}}[ \exp(-V_{B_{\Tc}^2(\theta), R, h+1}^{\pi}(\tau_{h},o)/\lambda^*(\tau_h) ) ] \\
    & = \lambda^*(\tau_h) \log \frac{ \Eb_{\Tc_h^{\theta}}[\exp(-V_{B_{\Tc}^2(\theta), R, h+1}^{\pi}(\tau_{h},o)/\lambda^*(\tau_h) ) ] }{ \Eb_{\Tc_h^{\theta^*}}[\exp(-V_{B_{\Tc}^2(\theta), R, h+1}^{\pi}(\tau_{h},o)/\lambda^*(\tau_h) ) ] } \\
    &\overset{(a)}\leq \lambda^*(\tau_h) \frac{ \Eb_{\Tc_h^{\theta}}[\exp(-V_{B_{\Tc}^2(\theta), R, h+1}^{\pi}(\tau_h, o)/\lambda^*(\tau_h) ) ] - \Eb_{\Tc_h^{\theta^*}}[\exp(-V_{B_{\Tc}^2(\theta), R, h+1}^{\pi}(\tau_h, o)/\lambda^*(\tau_h) ) ] }{ \Eb_{\Tc_h^{\theta^*}}[\exp(-V_{B_{\Tc}^2(\theta), R, h+1}^{\pi}(\tau_h, o)/\lambda^*(\tau_h) ) ] } \\
    &\leq \frac{  \lambda^*(\tau_h) }{ \exp(-1/\lambda^*(\tau_h) )  }  \left\| \Tc_h^{\theta^*}(\cdot|\tau_h) -  \Tc_h^{\theta}(\cdot|\tau_h) \right\|_1 \\
    &\leq \max\left\{ \frac{\exp(\xi)}{\xi},  \lambda^*\exp(1/\lambda^*) \right\} \left\| \Tc_h^{\theta^*}(\cdot|\tau_h) -  \Tc_h^{\theta}(\cdot|\tau_h) \right\|_1,
\end{align*}
where $(a)$ follows from $\log(x)\leq x - 1$ for $x>0$, and $0< \lambda^* \leq \min_{h,\tau_h}\lambda^*(\tau_h)$.

In summary, the difference between robust values can be upper bounded by
\begin{align*}
    V_{B_{\Tc}^i(\theta^*), R}^{\pi} - V_{B_{\Tc}^i(\theta), R}^{\pi}  \leq C_B C_i \sum_{h=1}^H \Eb_{\theta^*}^{\pi}\left[ \left\| \Tc_h^{\theta^*}(\cdot|\tau_h) -  \Tc_h^{\theta}(\cdot|\tau_h) \right\|_1\right]\leq 2C_BC_i\| 
    \Db_{\theta^*}^{\pi} - \Db_{\theta}^{\pi} \|_1,
\end{align*}
where 
\[
C_i = \left\{
\begin{aligned}
    & 1,\quad i = 1, \\
    &\max\left\{ \frac{\exp(\xi)}{\xi},  \lambda^*\exp(1/\lambda^*) \right\},\quad i = 2.
\end{aligned}
\right.
\]

The opposite side of the inequalities of the lemma follows the same argument.

\if{0}
If there is an $f: \mathcal{O}^{h+1}\times\mathcal{A}^h \rightarrow [0,1]$ such that 
\[ 0 = \arg\sup_{\lambda\geq 0} \left\{ -\lambda \log \Eb_{\Tc_h^{\theta}}[\exp(-f(\tau_h, o)/\lambda ) ] - \lambda \xi \right\},\]

then we have $\sup_{\lambda\geq 0} \left\{ -\lambda \log \Eb_{\Tc_h^{\theta}}[\exp(-f(\tau_h, o)/\lambda ) ] - \lambda \xi \right\} = \min_o f(\tau_h,o)$, which implies 
\[ \left| \sup_{\lambda\geq 0} \left\{ -\lambda \log \Eb_{\Tc_h^{\theta}}[\exp(-f(\tau_h, o)/\lambda ) ] - \lambda \xi \right\}   - \sup_{\lambda \geq 0 } \left\{ - \lambda \log \Eb_{\Tc_h^{\theta^*}}[ \exp(-f(\tau_h, o)/\lambda ) ] -  \lambda \xi \right\} \right|  = 0. \]
\fi

\end{proof}

\section{MLE Analysis}\label{sec:general MLE analysis}

In this section, we provide the estimation guarantee for MLE oracle in \Cref{lemma:offline robust MLE guarantee}. 

We first present a sequence of supporting lemmas in the following. Notably, the first lemma states that after dataset distillation, we still have sufficient offline samples.

\begin{lemma}[Sufficient good samples]\label{lemma:good samples enough}
    Let $p_{\min}<  \frac{\delta}{|\mathcal{O}|^{2H}|\mathcal{A}|^{2H}}$, then, with probability at least $1-\delta$, we have  $|\Dc^{\mathrm{g}}|\geq N/2$.
\end{lemma}

\begin{proof}

Let $L^* = H\log(|\mathcal{O}||\mathcal{A}|)$. Then, we have
\begin{align*}
    \Eb&\left[\exp\left(\sum_{n=1}^N(-\log\Db_{\theta^*}^{\rho}(\tau_H^n) - L^*)\right)\right] \\
    & = \Eb\left[\prod_{n=1}^N \frac{ 1 }{ \Db_{\theta^*}^{\rho}(\tau_H) e^{L^*}} \right] \\
    & = \prod_{n=1}^N \frac{|\mathcal{O}|^H|\mathcal{A}|^H}{e^{L^*}} \\
    & = 1.
\end{align*}

Thus, by Chernoff bound, we have that with probability $1-\delta$,
\[ \min_{\theta}\mathcal{L}(\theta|\Dc) \leq \frac{1}{N} \sum_{n=1}^N-\log\Db_{\theta^*}^{\rho}(\tau_H^n)\leq   L^* + \frac{1}{N}\log\frac{1}{\delta}. \]
    
Recall that $\hat{\theta} = \min_{\theta}\mathcal{L}(\theta|\Dc).$ Let $\mathcal{D}^{\mathrm{b}} = \Dc/\Dc^{\mathrm{g}} = \{\tau_H\in\mathcal{D}, \Db_{\hat{\theta}}^{\rho}(\tau_H) < p_{\min} \}$.  Thus, we further have
\begin{align*}
    N L^*  + \log\frac{1}{\delta} & \geq \sum_{n=1}^N - \log\Db_{\hat{\theta}}^{\rho}(\tau_H^n) \\
    &\geq \sum_{\tau_h \in \Dc^{\mathrm{b}}} -\log\Db^{\rho}_{\hat{\theta}}(\tau_H) \\
    &\geq -|\Dc^{\mathrm{b}}|\log p_{\min},
\end{align*}
which implies
\[ |\Dc^{\mathrm{b}}| \leq \frac{ NL^* + \log(1/\delta) }{\log\frac{1}{p_{\min}}} < \frac{N}{2}. \]

\end{proof}

\begin{lemma}\label{proposition: log likelihood of true model is large}
    Fix $\varepsilon<\frac{1}{N}$. Let $\bar{\Theta}_{\varepsilon}$ with size $\mathcal{N}_{\varepsilon}(\Theta)$ consist of all $\varepsilon$-brackets who can cover $\Theta$. With probability at least $1-\delta$, for any $[\underline{\theta}, \bar{\theta}]\in\bar{\Theta}_{\varepsilon}$, the following two inequalities hold:
    \[\Ec_{\log} = \left\{
    \begin{aligned}
        &\forall \bar{\theta}\in\bar{\Theta}_{\varepsilon},~~  \sum_h\sum_{ \tau_h \in\Dc_h^{\mathrm{g}}}\log \Db_{\bar{\theta}}^{\rho}(\tau_h)  - 3\log\frac{\mathcal{N}_{\varepsilon}(\Theta)}{\delta}\leq \sum_{h}\sum_{ \tau_h \in\Dc_h^{\mathrm{g}}} \log \Db_{\theta^*}^{\pi}(\tau_h), \\
        &\forall \bar{\theta}\in\bar{\Theta}_{\varepsilon},~~ \sum_{ \tau_H \in\Dc^{\mathrm{g}}}\log \Db_{\bar{\theta}}^{\rho}(\tau_H)  - 3\log\frac{\mathcal{N}_{\varepsilon}(\Theta)}{\delta}\leq \sum_{ \tau_H \in\Dc^{\mathrm{g}} }\log \Db_{\theta^*}^{\rho}(\tau_H).   
    \end{aligned}
    \right.
    \]
\end{lemma}

\begin{proof} We start with the first inequality. Suppose the data in $\Dc_h^{\mathrm{g}}$ is indexed by $t$. Then,
\begin{align*}
    \Eb&\left[ \exp\left(\sum_h\sum_{ \tau_h \in\Dc_h^{\mathrm{g}}} \log\frac{ \Db_{\bar{\theta}}^{\rho}( \tau_h ) }{\Db_{\theta^*}^{\rho}(\tau_h )} \right)\right] \\
    & = \Eb\left[ \prod_{t\leq |\Dc_h^{\mathrm{g}}|}\prod_h \frac{ \Db_{\bar{\theta}}^{\rho}(\tau_h^t) }{ \Db_{ \theta^*}^{\rho}(\tau_h^t) } \right] \\
    & = \Eb\left[ \prod_{t\leq |\Dc_h^{\mathrm{g}}|-1}\prod_h \frac{ \Db_{\bar{\theta}}^{\rho}(\tau_h^t) }{ \Db_{ \theta^*}^{\rho}(\tau_h^t) } \Eb\left[ \frac{ \Db_{\bar{\theta}}^{\rho}(\tau_h^{|\Dc_h^{\mathrm{g}}|}) }{ \Db_{ \theta^*}^{\rho}(\tau_h^{|\Dc_h^{\mathrm{g}}|}) } \right] \right] \\
    & = \Eb\left[ \prod_{t\leq |\Dc_h^{\mathrm{g}}|-1}\prod_h \frac{ \Db_{\bar{\theta}}^{\rho}(\tau_h^t) }{ \Db_{ \theta^*}^{\rho}(\tau_h^t) } \prod_h\sum_{\tau_h}   \Db_{\bar{\theta}}^{\rho}(\tau_h )   \right] \\
    &\overset{(a)}\leq \left(1 + \varepsilon\right)^H \Eb\left[ \prod_{t\leq |\Dc_h^{\mathrm{g}}|-1}\prod_h \frac{ \Db_{\bar{\theta}}^{\rho}(\tau_h^t) }{ \Db_{ \theta^*}^{\rho}(\tau_h^t) }   \right] \\
    &\leq \left(1+\varepsilon\right)^{|\Dc^{\mathrm{g}}|} \\
    &\overset{(b)}\leq e,
\end{align*}
where $(a)$ follows because  $\sum_{\tau_h}|\Db_{\bar{\theta}}^{\rho}(\tau_h) - \Db_{\theta}^{\rho}(\tau_h) | \leq \|\Db_{\bar{\theta}}^{\rho} - \Db_{\theta}^{\rho}\|_1 \leq \varepsilon$, and $(b)$ follows because $\varepsilon\leq \frac{1}{N}\leq \frac{1}{|\Dc^{\mathrm{g}}|}$.

By the Chernoff bound and the union bound over $\bar{\Theta}_{\epsilon}$, with probability at least $1-\delta$, we have
\begin{align*}
    \forall \bar{\theta}\in\bar{\Theta}_{\epsilon},~~ \sum_h\sum_{ \tau_h \in\Dc_h^{\mathrm{g}}}\log\frac{ \Db_{\bar{\theta}}^{\rho}( \tau_h ) }{\Db_{\theta^*}^{\rho}(\tau_h )} \leq 3\log\frac{\left|\bar{\Theta}_{\epsilon}\right|}{\delta},
\end{align*}
which  yields the first result of this proposition.

To show the second inequality, we follow an argument similar to that for the first inequality. We have
\begin{align*}
    \Eb  \left[ \exp\left( \sum_{ \tau_H \in\Dc^{\mathrm{g}}} \log\frac{ \Db_{ \theta }^{\rho}(\tau_H) }{\Db_{ \theta^* }^{\rho}(\tau_H)}  \right)\right] &\leq \Eb\left[ \exp\left( \sum_{ \tau_H \in\Dc^{\mathrm{g}}} \log \frac{ \Db_{ \bar{\theta} }^{\rho}(\tau_H) }{\Db_{ \theta^* }^{\rho}(\tau_H)}  \right)\right]  \\
    &\overset{(a)}\leq (1+\varepsilon)^{N}\leq e,
\end{align*}
where $(a)$ follows from the tower rule of the expectation and because $\sum_{\tau_H}\Db_{\bar{\theta}}^{\rho}(\tau_H)\leq 1+\varepsilon$.

Thus, with probability at least $1-\delta$, for any $\bar{\theta}\in\Theta$, the following inequality holds
\begin{align*}
    \sum_{ \tau_H \in\Dc^{\mathrm{g}} } \log\frac{ \Db_{ \theta }^{\rho}(\tau_H) }{\Db_{ \theta^* }^{\rho}(\tau_H)} \leq 3\log\frac{\mathcal{N}_{\varepsilon}(\Theta)}{\delta},
\end{align*}
which completes the proof.
\end{proof}


\begin{proposition}\label{proposiiton: empirical distance less than log likelihood difference}
Fix $p_{\min}$ and $\varepsilon\leq \frac{p_{\min}}{N}$.  
Let $\Theta_{\min}  = \{\theta: \forall h,  \tau_h \in\Dc_h^{\mathrm{g}}, ~~ \Db_{\theta}^{\rho}(\tau_h)  \geq p_{\min}\}$. Consider the following event
    \begin{align*}
        \Ec_{\omega} &= \left\{ \forall  \theta\in\Theta_{\min},  ~~   \sum_h\sum_{\tau_h\in\Dc_h^{\mathrm{g}}}  \left\| \Db_{ \hat{\theta} }^{\rho} (\cdot|\tau_h) - \Db_{ \theta^* }^{\rho} (\cdot|\tau_h) \right\|_1^2   \leq  6\sum_h\sum_{ \tau_H \in\Dc_h^{\mathrm{g}} }\log\frac{\Db_{\theta^*}^{ \rho }(\tau_H)}{\Db_{ \theta}^{ \rho  }(\tau_H)} + 31\log\frac{\mathcal{N}_{\varepsilon}(\Theta)}{\delta}  \right\}.
    \end{align*}
Then, $\Pb\left(\Ec_{\omega} \right) \geq 1- \delta.$
\end{proposition}
\begin{proof}
We start with a general upper bound on the total variation distance between two conditional distributions. Note that for any $\theta,\theta'\in\Theta\cup\bar{\Theta}_{\epsilon}$ and fixed $(\tau_h,\pi)$, we have
\begin{align*}
    &\left\| \Db_{\theta}^{\rho}(\omega_h|\tau_h ) - \Db_{\theta'}^{\rho}(\omega_h|\tau_h )\right\|_1 \\
    &\quad  = \sum_{\omega_h} \bigg| \frac{\Db_{\theta'}^{ \pi }(\omega_h,\tau_h ) \Db_{\theta}^{\rho }(\tau_h ) - \Db_{\theta}^{ \rho }(\omega_h,\tau_h )\Db_{\theta'}^{\rho }(\tau_h )  }{\Db_{\theta}^{\pi }(\tau_h )\Db_{\theta'}^{\rho }(\tau_h )} \bigg|\\
    &\quad = \sum_{\omega_h} \left| \frac{ \left(\Db_{\theta'}^{ \rho }(\omega_h,\tau_h ) - \Db_{\theta}^{ \rho }(\omega_h,\tau_h ) \right) \Pb_{\theta}^{\rho }(\tau_h ) + \Db_{\theta}^{ \rho }(\omega_h,\tau_h ) \left(\Db_{\theta}^{\pi }(\tau_h ) -\Db_{\theta'}^{\rho }(\tau_h )  \right) }{\Db_{\theta}^{\rho }(\tau_h ) \Db_{\theta'}^{\rho }(\tau_h )} \right|\\
    &\quad \leq \frac{|\Db_{\theta}^{\rho }(\tau_h ) - \Db_{\theta'}^{\pi }(\tau_h )|}{\Db_{\theta'}^{\pi }(\tau_h ) } + \frac{1}{\Db_{\theta'}^{\rho }(\tau_h )}\sum_{\omega_h} \left|\left(\Db_{\theta'}^{ \rho }(\omega_h,\tau_h ) - \Db_{\theta}^{ \rho }(\omega_h,\tau_h ) \right)\right|\\
    &\quad \leq \frac{2}{\Db_{\theta'}^{\rho }(\tau_h )}  \left\| \Db_{\theta}^{ \rho }  - \Db_{\theta'}^{ \rho }  \right\|_1.
\end{align*}

By symmetry, we also have
\begin{align*}
    \left\| \Db_{\theta}^{\rho}(\omega_h|\tau_h ) - \Db_{\theta'}^{\rho}(\omega_h|\tau_h )\right\|_1 \leq \frac{2}{ \max\left\{\Db_{\theta}^{\rho }(\tau_h ), \Db_{\theta'}^{\rho }(\tau_h ) \right\} } \left\| \Db_{\theta}^{ \rho }  - \Db_{\theta'}^{ \rho }  \right\|_1.
\end{align*}

We replace $\theta'$ by a $\bar{\theta} $ where $[\underline{\theta},\bar{\theta}] \in\bar{\Theta}_{\varepsilon}$  is an $\varepsilon$-bracket of $\theta$ (recall \Cref{def:optimistic net}), i.e. $
    \left\| \Db_{\theta}^{ \rho }  - \Db_{\bar{\theta}}^{ \rho }  \right\|_1 \leq \varepsilon$, and $\Db_{\bar{\theta}}^{\rho}(\tau_h)\geq \Db_{\theta}^{\rho}(\tau_h)$, $\forall \tau_h$. 
Then, due the construction of $\Theta_{\min}^k$, we have
\[\forall  \tau_h \in\Dc_h^{\mathrm{g}},~~ \left\| \Db_{\theta}^{\rho}(\omega_h|\tau_h ) - \Db_{\theta'}^{\rho}(\omega_h|\tau_h )\right\|_1 \leq \frac{2\varepsilon}{p_{\min}}\leq \frac{2}{N}, \]
which implies
\begin{align*}
    \sum_h\sum_{ \tau_h \in\Dc_h^{\mathrm{g}}} & \left\| \Db_{\theta}^{\rho}(\omega_h|\tau_h ) - \Db_{\theta'}^{\rho}(\omega_h|\tau_h )\right\|_1^2 \\
    &\overset{(a)}\leq \sum_h\sum_{ \tau_h \in\Dc_h^{\mathrm{g}} }  2 \left\| \Db_{\theta}^{\rho}(\omega_h|\tau_h ) - \Pb_{\bar{\theta}}^{\rho}(\omega_h|\tau_h )\right\|_1^2 +  2 \left\| \Db_{\bar{\theta}}^{\rho}(\omega_h|\tau_h ) - \Db_{\theta^*}^{\rho}(\omega_h|\tau_h )\right)\|_1^2 \\
    &\leq \frac{ 8 }{N} + 2 \sum_h\sum_{ \tau_h \in\Dc_h^{\mathrm{g}} }   \mathtt{D}_{\TV}^2\left( \Pb_{\bar{\theta}}^{\pi}(\omega_h|\tau_h ), \Pb_{\theta^*}^{\pi}(\omega_h|\tau_h )\right).
\end{align*}
Here $(a)$ follows because the total variation distance satisfies the triangle inequality and  $(a+b)^2\leq 2a^2 + 2b^2$.

Moreover, note that 
\begin{align*}
     &\left\| \Db_{\bar{\theta}}^{\rho}(\omega_h|\tau_h ) - \Db_{\theta^*}^{\rho}(\omega_h|\tau_h )\right\|_1^2 \\
    &\quad \overset{(a)} \leq 4(2 + 2/(N))\mathtt{D}^2_{\mathtt{H}}  \left( \Db_{\bar{\theta}}^{\rho}(\omega_h|\tau_h ), \Db_{\theta^*}^{\rho}(\omega_h|\tau_h ) \right) \\
    & \quad  \leq  6\left( 1+\frac{1}{N} - \mathop{\Eb}_{\omega_h\sim \Db_{\theta^*}^{\pi} } \sqrt{\frac{\Db_{\bar{\theta}}^{\rho}(\omega_h|\tau_h )}{\Db_{\theta^*}^{\rho}(\omega_h|\tau_h )}} \right)\\
    &\quad  \overset{(b)}\leq - 6 \log \mathop{\Eb}_{\omega_h\sim \Db_{\theta^*}^{\rho}(\cdot|\tau_h ) } \sqrt{\frac{\Db_{\bar{\theta}}^{\rho}(\omega_h|\tau_h )}{\Db_{\theta^*}^{\rho}(\omega_h|\tau_h )}} + \frac{6}{N},
\end{align*}
where $(a)$ is due to \Cref{lemma:TV and hellinger} and $(b)$ follows because $1-x\leq -\log x$ for any $x>0$.

Thus, the summation of the total variation distance between conditional distributions conditioned on $ \tau_h \in\Dc_h^{\mathrm{g}}$  can be upper bounded by 
\begin{align*}
    \sum_h\sum_{ \tau_h \in\Dc_h^{\mathrm{g}}} & \left\| \Db_{\bar{\theta}}^{\rho}(\omega_h|\tau_h ) - \Db_{\theta^*}^{\rho}(\omega_h|\tau_h )\right\|_1^2 \\
    &\leq \frac{18}{N} - 12\sum_h\sum_{ \tau_h \in\Dc_h^{\mathrm{g}}} \log \mathop{\Eb}_{\omega_h\sim \Db_{\theta^*}^{\rho}(\cdot|\tau_h ) } \sqrt{\frac{\Db_{\bar{\theta}}^{\rho}(\omega_h|\tau_h )}{\Db_{\theta^*}^{\rho}(\omega_h|\tau_h )}}.
\end{align*}

In addition, by only taking expectation over $\omega_h$, we have
\begin{align*}
    &\mathop{\Eb}_{ \substack{\forall h, \tau_h \in\Dc_h^{\mathrm{g}}, \\ \omega_h \sim \Db_{\theta^*}^{\rho}(\cdot|\tau_h) }} \left[\exp\left( \frac{1}{2}\sum_h\sum_{(\omega_h,\tau_h )\in\Dc_h^{\mathrm{g}}} \log\frac{\Db_{\bar{\theta}}^{\rho}(\omega_h |\tau_h )}{\Db_{\theta^*}^{\rho}(\omega_h |\tau_h )} - \sum_h\sum_{ \tau_h \in\Dc_h^{\mathrm{g}}} \log \mathop{\Eb}_{\omega_h\sim \Db_{\theta^*}^{\rho}(\cdot|\tau_h ) } \sqrt{\frac{\Db_{\bar{\theta}}^{\rho}(\omega_h|\tau_h )}{\Db_{\theta^*}^{\rho}(\omega_h|\tau_h )}} \right)\right]\\
    & \quad\quad =  \frac{ \mathop{\Eb}_{ \substack{\forall h, \tau_h \in\Dc_h^{\mathrm{g}},\\ \omega_h \sim \Db_{\theta^*}^{\rho}(\cdot|\tau_h) } } \left[\prod_h\prod_{(\omega_h,\tau_h)\in\Dc_h^{\mathrm{g}}} \sqrt{\frac{ \Db_{\bar{\theta}}^{\rho}(\omega_h |\tau_h ) }{\Db_{\theta^*}^{\rho}(\omega_h |\tau_h )} }\right] }{\prod_h\prod_{ \tau_h \in\Dc_h^{\mathrm{g}}} \mathop{\Eb}_{\omega_h\sim \Db_{\theta^*}^{\rho}(\cdot|\tau_h ) } \left[\sqrt{\frac{\Db_{\bar{\theta}}^{\rho}(\omega_h|\tau_h )}{\Db_{\theta^*}^{\rho}(\omega_h|\tau_h )}} \right] }  = 1, 
\end{align*}
where the last equality is due to the conditional independence of $\omega_h \in \Dc_h^{\mathrm{g}}$ given $\tau_h\in\Dc_h^{\mathrm{g}}$.

Therefore, by the Chernoff bound, with probability $1-\delta$, we have

\begin{align*}
    -\sum_h\sum_{ \tau_h \in\Dc_h^{\mathrm{g}}} \log \mathop{\Eb}_{\omega_h\sim \Db_{\theta^*}^{\rho}(\cdot|\tau_h ) } \sqrt{\frac{\Db_{\bar{\theta}}^{\rho}(\omega_h|\tau_h )}{\Pb_{\theta^*}^{\rho}(\omega_h|\tau_h )}} \leq \frac{1}{2} \sum_h\sum_{(\omega_h,\tau_h ) \in\Dc_h^{\mathrm{g}} } \log \frac{\Db_{\theta^*}^{\rho}(\omega_h |\tau_h )}{\Db_{\bar{\theta}}^{\rho}(\omega_h |\tau_h )} + \log\frac{1}{\delta}.
\end{align*}

Taking the union bound over $\bar{\Theta}_{\epsilon}$, and rescaling $\delta$, we have, with probability at least $1-\delta$, the following inequality holds:
\begin{align*}
    \sum_h\sum_{ \tau_h \in\Dc_h^{\mathrm{g}}} & \left\| \Db_{\theta}^{\rho}(\omega_h|\tau_h ) - \Db_{\theta^*}^{\rho}(\omega_h|\tau_h )\right\|_1^2 \\
    &\leq \frac{18 }{ N } +  6\sum_h\sum_{(\omega_h,\tau_h )\in\Dc_h^{\mathrm{g}}}\log\frac{\Db_{\theta^*}^{\rho}(\omega_h |\tau_h )}{\Db_{\bar{\theta}}^{\rho}(\omega_h |\tau_h )} + 12\log\frac{\mathcal{N}_{\varepsilon}(\Theta)}{\delta}\\
    &\leq 6\sum_h\sum_{(\omega_h, \tau_h ) \in\Dc_h^{\mathrm{g}}}\log\frac{\Db_{\theta^*}^{\rho}( \omega_h ,\tau_h )}{\Db_{\bar{\theta}}^{\rho}(\omega_h ,\tau_h )} + 6\sum_h\sum_{ \tau_h \in\Dc_h^{\mathrm{g}}}\log\frac{ \Db_{\bar{\theta}}^{\rho}( \tau_h ) }{\Db_{\theta^*}^{\rho}(\tau_h )} + 13\log\frac{\mathcal{N}_{\varepsilon}(\Theta)}{\delta}.  
\end{align*}

Note that, following from \Cref{proposition: log likelihood of true model is large}, with probability at least $1-\delta$, we have for any $ k\in[K]$,
\begin{align*}
 \sum_h\sum_{ \tau_h \in\Dc_h^{\mathrm{g}}}\log\frac{ \Db_{\bar{\theta}}^{\rho}( \tau_h ) }{\Db_{\theta^*}^{\rho}(\tau_h )} \leq 3\log\frac{\mathcal{N}_{\varepsilon}(\Theta)}{\delta}.
\end{align*}

Hence, combining with the optimistic property of $\bar{\theta}$ and rescaling $\delta$, we have that the following inequality holds with probability at least $1-\delta$:
\begin{align*}
    \forall \theta\in\Theta_{\min},~~ \sum_h\sum_{ \tau_h \in\Dc_h^{\mathrm{g}}} & \left\| \Db_{\theta}^{\rho}(\omega_h|\tau_h ) - \Db_{\theta^*}^{\rho}(\omega_h|\tau_h )\right\|_1^2  \leq  6\sum_h\sum_{ \tau_H \in\Dc_h^k}\log\frac{\Pb_{\theta^*}^{ \rho  }(\tau_H )}{\Pb_{ \theta}^{ \rho  }(\tau_H )} +  31\log\frac{\mathcal{N}_{\varepsilon}(\Theta)}{\delta},
\end{align*}
which yields the final result. 
\end{proof}

\begin{proposition}\label{proposition: hellinger distance less than log likelihood distance}
Fix $\varepsilon<\frac{1}{N^2}$.   Define the following event:
\begin{align*}
    \Ec_{\mathrm{H}} =  \left\{ \forall \theta\in\Theta, ~~   |\Dc^{\mathrm{g}}|\mathtt{D}_{\mathtt{H}}^2  ( \Db_{\theta}^{\rho}(\tau_H) , \Db_{\theta^*}^{\rho} (\tau_H) ) \leq \frac{1}{2}\sum_{ \tau_H \in\Dc^{\mathrm{g}}} \log\frac{ \Db_{ \theta^*}^{\pi}(\tau_H) }{\Db_{ \theta }^{\rho}(\tau_H)} +  2\log\frac{\mathcal{N}_{\varepsilon}(\Theta)}{\delta} \right\}.
\end{align*}
We have $\Pb(\Ec_{\pi}) \geq 1-\delta.$
\end{proposition}
\begin{proof}
First, by the construction of $\bar{\Theta}_{\varepsilon}$, for any $\theta$, let $\bar{\theta}$ satisfy $\sum_{\tau_H}\left|\Db_{ \theta}^{\rho}(\tau_H) -  \Db_{\bar{\theta}}^{\rho}(\tau_H) \right|\leq\varepsilon$. We translate the distance between $\theta$ and $\theta^*$ to the distance between $\bar{\theta}$ and $\theta^*$ as follows.
\begin{align*}
    \mathtt{D}_{\mathtt{H}}^2 & ( \Db_{\theta}^{\rho}  , \Db_{\theta^*}^{\rho}   ) \\
    & = 1 -  \sum_{\tau_H} \sqrt{ \Db_{\theta}^{\rho}(\tau_H) \Db_{\theta^*}^{\rho} (\tau_H)  }\\
    & = 1 -  \sum_{\tau_H} \sqrt{ \Db_{\bar{\theta}}^{\rho}(\tau_H) \Db_{\theta^*}^{\rho} (\tau_H)  +  \left(\Db_{ \theta}^{\rho}(\tau_H) -  \Db_{\bar{\theta}}^{\rho}(\tau_H) \right)\Db_{\theta^*}^{\rho} (\tau_H)} \\
    &\overset{(a)} \leq 1 - \sum_{\tau_H} \sqrt{ \Db_{\bar{\theta}}^{\rho}(\tau_H) \Db_{\theta^*}^{\rho} (\tau_H)  } + \sum_{\tau_H}\sqrt{ \left|\Db_{ \theta}^{\rho}(\tau_H) -  \Db_{\bar{\theta}}^{\rho}(\tau_H) \right| \Db_{\theta^*}^{\rho} (\tau_H)} \\
    &\overset{(b)} \leq - \log \mathop{\Eb}_{\tau_H\sim\Db_{\theta^*}^{\rho}(\cdot)} \sqrt{ \frac{ \Db_{ \bar{\theta} }^{\rho}(\tau_H) }{ \Db_{\theta^*}^{\rho}(\tau_H) } } + \sqrt{\sum_{\tau_H}\left|\Db_{ \theta}^{\rho}(\tau_H) -  \Db_{\bar{\theta}}^{\rho}(\tau_H) \right| } \\
    & \leq  - \log \mathop{\Eb}_{\tau_H\sim\Db_{\theta^*}^{\rho}(\cdot)} \sqrt{ \frac{ \Db_{ \bar{\theta} }^{\rho}(\tau_H) }{ \Db_{\theta^*}^{\rho}(\tau_H) } } + \sqrt{\varepsilon},
\end{align*}
where $(a)$ follows because $\sqrt{a+b} \geq\sqrt{a} - \sqrt{|b|}$ if $a>0$ and $a+b>0$, and $(b)$ follows from the Cauchy's inequality and the fact that $1-x\leq - \log x$.

Hence, in order to upper bound $|\Dc^{\mathrm{g}}|\mathtt{D}_{\mathtt{H}}^2 ( \Db_{\theta}^{\rho} , \Db_{\theta^*}^{\rho}  ) $, it suffices to upper bound $ - \log \mathop{\Eb}_{\tau_H\sim\Db_{\theta^*}^{\rho}(\cdot)} \sqrt{ \frac{ \Db_{ \bar{\theta} }^{\rho}(\tau_H) }{ \Db_{\theta^*}^{\rho}(\tau_H) } }.$ To this end, we observe that,
\begin{align*}
 \Eb&\left[\exp\left( \frac{1}{2} \sum_{ \tau_H \in\Dc^{\mathrm{g}}} \log\frac{ \Db_{\bar{\theta}}^{\rho}(\tau_H) }{\Db_{\theta^*}^{\pi}(\tau_H)}  -  |\Dc^{\mathrm{g}}|\log \mathop{\Eb}_{\tau_H\sim\Db_{\theta^*}^{\rho}(\cdot)} \sqrt{ \frac{ \Db_{\theta}^{\rho}(\tau_H) }{ \Db_{\theta^*}^{\rho}(\tau_H) } } \right)\right] \\
    &\quad \overset{(a)}= \frac{ \Eb\left[\prod_{ \tau_H \in\Dc^{\mathrm{g}} } \sqrt{ \frac{ \Db_{\theta}^{\rho}(\tau_H) }{ \Db_{\theta^*}^{\rho}(\tau_H) } } \right] }{ \Eb\left[\prod_{ \tau_H \in\Dc^{\mathrm{g}}} \sqrt{ \frac{ \Db_{\theta}^{\rho}(\tau_H) }{ \Db_{\theta^*}^{\rho}(\tau_H) } } \right] }  = 1.
\end{align*}

Then, by the Chernoff bound, 
we have
\begin{align*}
    \Pb&\left( \frac{1}{2} \sum_{ \tau_H \in\Dc^{\mathrm{g}}} \log\frac{ \Db_{\bar{\theta}}^{\rho}(\tau_H) }{\Db_{\theta^*}^{\rho}(\tau_H)}  -  |\Dc^{\mathrm{g}}|\log \mathop{\Eb}_{\tau_H\sim\Db_{\theta^*}^{\rho}(\cdot)} \sqrt{ \frac{ \Db_{\theta}^{\rho}(\tau_H) }{ \Db_{\theta^*}^{\rho}(\tau_H) } }  \geq \log\frac{1}{\delta} \right) \\
    &\leq \delta.
\end{align*}

Finally, rescaling $\delta$ to $\delta/(\mathcal{N}_{\varepsilon}(\Theta))$ and taking the union bound over $\bar{\Theta}_{\epsilon}$, we conclude that, with probability at least $1-\delta$, ,
\begin{align*}
    \forall \theta\in\Theta, ~~& |\Dc^{\mathrm{g}}|\mathtt{D}_{\mathtt{H}}^2  ( \Db_{\theta}^{\rho}(\tau_H) , \Db_{\theta^*}^{\rho} (\tau_H) ) \\
    &\leq N\sqrt{\varepsilon} + \frac{1}{2} \sum_{ \tau_H \in\Dc^{\mathrm{g}}} \log\frac{ \Db_{\theta^*}^{\rho}(\tau_H) }{\Db_{\bar{\theta}}^{\rho}(\tau_H)}  + \log\frac{ \mathcal{N}_{\varepsilon}(\Theta)}{\delta} \\
    &\overset{(a)}\leq  \frac{1}{2} \sum_{ \tau_H \in\Dc^{\mathrm{g}}} \log\frac{ \Db_{\theta^*}^{\rho}(\tau_H) }{\Db_{ \theta }^{\rho}(\tau_H)} + 2\log\frac{\mathcal{N}_{\varepsilon}(\Theta)}{\delta}, 
\end{align*}
where $(a)$ follows because $\varepsilon\leq \frac{1}{N^2}$.
\end{proof}
 
Now, we are ready to provide the MLE guarantee of under the behavior policy $\rho$.
 
\begin{lemma}[MLE guarantee]\label{lemma:offline robust MLE guarantee}
Let $p_{\min}<\delta/(|\mathcal{O}||\mathcal{A}|)^{2H}$. With probability at least $1-\delta$, for any parameter $\theta$ satisfying $\mathcal{L}(\theta|\Dc)\leq \min_{\theta'}\mathcal{L}(\theta'|\Dc) + \beta/N$, we have the following inequalities.
 \begin{align*}
        &   \sum_h\sum_{\tau_h\in\Dc_h^{\mathrm{g}}}  \left\| \Db_{ \theta }^{\rho} (\cdot|\tau_h) - \Db_{ \theta^* }^{\rho} (\cdot|\tau_h) \right\|_1^2  \leq 13\beta,  \\
        & \mathtt{D}_{\mathtt{H}}^2\left( \Db_{ \theta }^{\rho} , \Db_{ \theta^* }^{\rho}   \right) \leq \frac{ 4\beta }{N},
    \end{align*}
    where $\beta_1 = 31\log\frac{3\mathcal{N}_{\varepsilon}(\Theta)}{\delta}$.
\end{lemma}

\begin{proof}

Consider the three events $\Ec_{\log}$, $\Ec_{\omega}$, and $\Ec_{\mathrm{H}}$, defined in \Cref{proposition: log likelihood of true model is large}, \Cref{proposiiton: empirical distance less than log likelihood difference}, \Cref{proposition: hellinger distance less than log likelihood distance}, respectively. Let $\Ec_{o} = \Ec_{\log}\cap \Ec_{\omega}\cap\Ec_{\mathrm{H}}$. Then, we have $\Pb(\Ec_{o})\geq 1-\delta.$ The following proof is under the case when $\Ec_o$ happens.

First, due to the condition that $\mathcal{L}(\theta|\Dc)\leq \min_{\theta'}\mathcal{L}(\theta'|\Dc) + \beta/N$, we have
\begin{align*}
    \sum_h\sum_{ \tau_H \in\Dc_h^{\mathrm{g}} } & \log\frac{\Db_{\theta^*}^{ \rho }(\tau_H)}{\Db_{ \theta }^{ \rho  }(\tau_H)} \\
    & = \sum_{ \tau_H \in\Dc  }  \log \Db_{\theta^*}^{ \rho }(\tau_H) - \sum_{ \tau_H \in\Dc }\log\Db_{ \theta }^{ \rho  }(\tau_H) + \sum_{ \tau_H \notin\Dc^{\mathrm{g}} }\log\Db_{ \theta }^{ \rho  }(\tau_H) - \sum_{ \tau_H \notin\Dc^{\mathrm{g}} }\log\Db_{ \theta^* }^{ \rho  }(\tau_H)\\
    &\leq \sum_{ \tau_H \notin\Dc^{\mathrm{g}} }\log\Db_{ \theta }^{ \rho  }(\tau_H) - \sum_{ \tau_H \notin\Dc^{\mathrm{g}} }\log\Db_{ \theta^* }^{ \rho  }(\tau_H) + \beta \\
    & \overset{(a)}\leq  \sum_{ \tau_H \notin\Dc^{\mathrm{g}} }\log\Db_{ \bar{\theta} }^{ \rho  }(\tau_H) - \sum_{ \tau_H \notin\Dc^{\mathrm{g}} }\log\Db_{ \theta^* }^{ \rho  }(\tau_H) + \beta \\
    & \overset{(b)}\leq 2\beta,
\end{align*}
where $(a)$ is due to the definition of $\bar{\theta}$, and $(b)$ follows from \Cref{proposition: log likelihood of true model is large}.

Then, due to \Cref{proposiiton: empirical distance less than log likelihood difference}, we have

\begin{align*}
     \sum_h\sum_{\tau_h\in\Dc_h^{\mathrm{g}}}  \left\| \Db_{ \theta }^{\rho} (\cdot|\tau_h) - \Db_{ \theta^* }^{\rho} (\cdot|\tau_h) \right\|_1^2  &\leq  6\sum_h\sum_{ \tau_H \in\Dc_h^{\mathrm{g}} }\log\frac{\Db_{\theta^*}^{ \rho }(\tau_H)}{\Db_{ \theta}^{ \rho  }(\tau_H)} + 31\log\frac{\mathcal{N}_{\varepsilon}(\Theta)}{\delta} \\
     &\leq 12\beta + \beta\\
     &\leq13\beta.
\end{align*}

Similarly, due to \Cref{proposition: hellinger distance less than log likelihood distance}

\begin{align*}
    \mathtt{D}_{\mathtt{H}}^2\left( \Db_{ \theta }^{\rho} , \Db_{ \theta^* }^{\rho}   \right) & \leq \frac{1}{|\Dc^{\mathrm{g}}|} \left(  \sum_h\sum_{ \tau_H \in\Dc_h^{\mathrm{g}} } \frac{1}{2}\log\frac{\Db_{\theta^*}^{ \rho }(\tau_H)}{\Db_{ \theta}^{ \rho  }(\tau_H)} + 2\log\frac{\mathcal{N}_{\varepsilon}(\Theta)}{\delta} \right) \\
    &\leq \frac{1}{|\Dc^{\mathrm{g}}|} \left( \beta + \beta\right)\\
    &\leq  4\beta / N,
\end{align*}
where the last inequality follows from \Cref{lemma:good samples enough}.
    
\end{proof}

\if{0}
\begin{proof}
We index $\omega_{h-1} = (o_{h},a_h,\ldots,o_H,a_H)$ by $i$, and $\tau_{h-1}$ by $j$. In addition, we denote $\hat{\mathbf{m}}(\omega_{h})^{\top}\left(\hat{\Mbf}_h (o_h,a_h) - \Mbf_h^*(o_h,a_h) \right)  $ by $w_{i}^{\top}$, denote $ \bar{ \psi}^* (\tau_{h-1})$ by $x_{j}$, and denote $\pi(\omega_{h-1}|\tau_{h-1})$ by $\pi_{i|j}$. Then, we have
    \begin{align*}
        \sum_{\tau_H}& \left| \hat{\mathbf{m}} (\omega_h)^{\top} \left(\hat{\Mbf}_h (o_h,a_h) - \Mbf_h^*(o_h,a_h) \right) \psi^* (\tau_{h-1}) \right| \pi(\tau_H) \\
        & =  \sum_{i }\sum_{j }|w_{i}^{\top}x_{j}|\pi_{i|j}\Pb_{ \theta^* }^{\pi}(j)\\
        & =  \sum_{j }  \sum_{i } (\pi_{i|j}\cdot\mathtt{sgn}(w_{i}^{\top}x_{j})\cdot w_{i})^{\top} x_{j}\cdot\Pb_{\theta^*}^{\pi}(j)\\
        & =  \sum_{j} \left( \sum_{i} \pi_{i|j} \cdot \mathtt{sgn}(w_{i}^{\top}x_{j} )\cdot w_{i} \right)^{\top} x_{j}\cdot\Pb_{ \theta^* }^{\pi}(j)\\
        &\leq \mathop{\Eb}_{j\sim\Pb_{\theta^*}^{\pi}}\left[ \big\| x_j \big\|_{\Lambda_{h-1}^{-1}}  \sqrt{ \left\|\sum_{i} \pi_{i|j} \cdot \mathtt{sgn}(w_{i}^{\top}x_{j} )\cdot w_{i} \right\|_{ \Lambda_{h-1} }^2 } \right],
    \end{align*}
where $\Lambda_{h-1} = \lambda_0 I + \frac{K}{H} \Eb_{\tau_{h-1}\sim\Pb_{\theta^*}^{\pi^b}} \left[ \bar{\psi}^*(\tau_{h-1}) \bar{\psi}^*(\tau_{h-1})^{\top} \right]$ and $\lambda_0$ will be determined later. We fix $\tau_{h-1} = j_0$ and aim to analyze the coefficient of $\|x_{j_0}\|_{\Lambda_{h-1}^{-1}}$. We have
\begin{align*}
    &\left\| \sum_{i} \pi_{i|j_0} \cdot \mathtt{sgn}(w_{i}^{\top}x_{j_0} )\cdot w_{i} \right\|_{\Lambda_{h-1} }^2\\
    &\quad =  \underbrace{\lambda_0 \left\| \sum_{i} \pi_{i|j_0} \cdot \mathtt{sgn}(w_{i}^{\top}x_{j_0} )\cdot w_{i} \right\|_2^2}_{I_1} + \underbrace{\frac{K}{H}  \mathop{\Eb}_{j\sim\Pb_{\theta^*}^{\pi^b}} \left[ \left(\sum_{i} \pi_{i|j_0} \cdot \mathtt{sgn}(w_{i}^{\top}x_{j_0} )\cdot w_{i}\right)^{\top} x_{j} \right]^2  }_{I_2}.
\end{align*}

For the first term $I_1$, we have
\begin{align*}
    \sqrt{I_1}  & = \sqrt{\lambda_0} \max_{x\in\mathbb{R}^{d_{h-1}}: \|x\|_2=1} \left|\sum_{i} \pi_{i|j_0}\mathtt{sgn}(w_i^{\top}x_{j_0})w_i^{\top}  x\right|\\
    &\quad \leq \sqrt{\lambda_0}\max_{x\in\mathbb{R}^{d_{h-1}}:\|x\|_2=1} \sum_{\omega_{h-1}} \left| \mathbf{m}^*(\omega_h)^{\top} \left( \hat{\Mbf}_h (o_h,a_h) - \Mbf_h^*(o_h,a_h) \right)   x \right| \pi(\omega_{h-1}|j_0)\\
    &\quad \leq \sqrt{\lambda_0}\max_{x\in\mathbb{R}^{d_{h-1}}:\|x\|_2=1} \sum_{\omega_{h-1}}\left| \hat{\mathbf{m}} (\omega_h)^{\top} \hat{\Mbf}_h (o_h,a_h)   x \right| \pi(\omega_{h-1}|j_0) \\
    &\quad\quad + \sqrt{\lambda_0}\max_{x\in\mathbb{R}^{d_{h-1}}:\|x\|_2=1} \sum_{\omega_{h-1}}\left| \hat{\mathbf{m}}(\omega_h)^{\top}  \Mbf_h^*(o_h,a_h)   x \right| \pi(\omega_{h-1}|j_0) \\
    & \overset{(a)}\leq  \frac{\sqrt{d\lambda_0}}{\gamma}  + \frac{\sqrt{\lambda_0}}{\gamma}\max_{x\in\mathbb{R}^{d_{h-1}}:\|x\|_2=1} \sum_{o_h,a_h}\left\| \hat{\Mbf}_h (o_h,a_h)   x \right\|_1 \pi(a_h|o_h,j_0)  \\
    &\overset{(b)}\leq \frac{2Q_A\sqrt{d\lambda_0}}{\gamma^2},
\end{align*}
where $(a)$ follows from \Cref{assmp:well-condition}, and $(b)$ follows from \Cref{prop: well-condition PSR M}.

For the second term $I_2$,  we have
\begin{align*}
    I_2  &\leq \frac{K}{H}  \mathop{\Eb}_{\tau_{h-1}\sim\Pb_{\theta^*}^{\pi^b}} \left[\left( \sum_{\omega_{h-1}} \left| \hat{\mathbf{m}} (\omega_{h})^{\top}\left(\hat{\Mbf}_h (o_h,a_h) -  \Mbf_h^*(o_h,a_h) \right)\bar{\psi}^* (\tau_{h-1}) \right| \pi(\omega_{h-1}|j_0) \right)^2  \right] \\
    & \leq  \frac{K}{H} \mathop{\Eb}_{\tau_{h-1}\sim\Pb_{\theta^*}^{\pi^b}}  \left[ \left( \sum_{\omega_{h-1}} \left| \hat{\mathbf{m}}(\omega_h)^{\top} \left( \hat{\Mbf}_h (o_h,a_h) \bar{ \hat{\psi}} (\tau_{h-1}) - \Mbf_h^*(o_h,a_h)  \bar{\psi}^*(\tau_{h-1}) \right) \right| \pi(\omega_{h-1}|j_0) \right. \right.  \\
    &\hspace{3cm} +   \left. \left. \sum_{\omega_{h-1}}\left| \hat{\mathbf{m}} (\omega_h)^{\top}  \hat{\Mbf}_h(o_h,a_h) \left(\bar{\hat{\psi}}  (\tau_{h-1}) -  \bar{\psi}^*(\tau_{h-1}) \right) \right| \pi(\omega_{h-1}|j_0) \right)^2 \right]   \\
    & \overset{(a)}\leq \frac{K}{H} \mathop{\Eb}_{\tau_{h-1}\sim\Pb_{\theta^*}^{\pi^b} } \left[  \left( \frac{1}{\gamma} \sum_{o_h,a_h} \left\| \Pb_{ \hat{\theta} } (o_h|\tau_{h-1})\bar{\hat{\psi}}_{h} (\tau_{h}) - \Pb_{\theta}(o_h|\tau_{h-1}) \bar{\psi}_{h}^*(\tau_{h}) \right\|_1   \pi(a_h|o_h,j_0) \right. \right.\\
    & \hspace{3cm} +  \left. \left. \frac{1}{\gamma}       \left\| \bar{\hat{\psi}} (\tau_{h-1}) -  \bar{\psi}^*(\tau_{h-1})  \right\|_1  \right)^2  \right] \\
    & =  \frac{K}{H\gamma^2} \mathop{\Eb}_{\tau_{h-1}\sim\Pb_{\theta^*}^{\pi^b}}  \left[\left(  \sum_{o_h,a_h}  \sum_{\ell=1}^{|\mathcal{Q}_h|} \left| \Pb_{ \hat{\theta} }(\mathbf{o}_h^{\ell},o_h|\tau_{h-1},a_h,\mathbf{a}_h^{\ell}) - \Pb_{ \theta^* }(\mathbf{o}_h^{\ell},o_h|\tau_{h-1},a_h,\mathbf{a}_h^{\ell}) \right|\pi(a_h|o_h,j_0) \right.    \right.\\
    & \hspace{3cm} +  \left. \left.   \sum_{\ell=1}^{|\mathcal{Q}_{h-1}|}   \left| \Pb_{ \hat{\theta} }(\mathbf{o}_{h-1}^{\ell} | \tau_{h-1}, \mathbf{a}_{h-1}^{\ell} ) - \Pb_{ \theta^* } (\mathbf{o}_{h-1}^{\ell} | \tau_{h-1}, \mathbf{a}_{h-1}^{\ell} )  \right|    \right)^2 \right] \\
    &\leq \frac{K}{H\gamma^2} \mathop{\Eb}_{\tau_{h-1}\sim\Pb_{\theta^*}^{\pi^b}}  \left[  \left( \ \sum_{\mathbf{a}_{h-1}\in \mathcal{Q}_h^{\exp} }   \sum_{\omega_{h-1}^o }\left|\Pb_{ \hat{\theta} }(\omega^o_{h-1}|\tau_{h-1}, \mathbf{a}_{h-1} ) - \Pb_{\theta^*}(\omega^o_{h-1}|\tau_{h-1}, \mathbf{a}_{h-1} ) \right|    \right)^2  \right] \\
    &\overset{(b)}\leq \frac{ K }{ H \iota^2 \gamma^2} \mathop{\Eb}_{\tau_{h-1}\sim\Pb_{\theta^*}^{\pi^b}} \left[\mathtt{D}_{\TV}^2\left(    \Pb_{\hat{\theta} }^{ \pi^b }(\omega_{h-1} | \tau_{h-1}  ) , \Pb_{ \theta^* }^{ \pi^b } (\omega_{h-1} | \tau_{h-1}  )  \right) \right]\\
    &\overset{(c)}\leq \frac{64K}{H\iota^2\gamma^2}\mathtt{D}_{\mathtt{H}}^2 \left(\Pb_{\hat{\theta}}^{\pi^b}(\tau_H), \Pb_{\theta^*}^{\pi^b}(\tau_H) \right)\\
    &\overset{(d)}\lesssim \frac{ \hat{\beta}}{H\iota^2\gamma^2},
\end{align*} 
where $(a)$ follows from \Cref{assmp:well-condition} and \Cref{eqn:Mpsi is conditional probability}, $(b)$ follows because $\pi^b(\mathbf{a}_{h-1})\geq\iota$ for all $\mathbf{a}_{h-1}\in\mathcal{Q}_{h-1}^{\exp}$, $(c)$ follows from \Cref{lemma:TV and hellinger}, and $(d)$ follows from \Cref{lemma:offline robust MLE guarantee}.

Thus, by choosing $\lambda_0 = \frac{\gamma^4}{4Q_A^2d}$, we have
\begin{align*}
        \sum_{\tau_H}& \left| \hat{\mathbf{m}}(\omega_h)^{\top} \left(\hat{\Mbf}_h (o_h,a_h) - \Mbf_h^*(o_h,a_h) \right) \psi^* (\tau_{h-1}) \right| \pi(\tau_H) \\
        &\lesssim \sqrt{\frac{\hat{\beta}}{H\iota^2\gamma^2}} \mathop{\Eb}_{\tau_{h-1}\sim \Pb_{ \theta^* }^{\pi}}\left[  \left\|\bar{\psi}^* (\tau_{h-1})\right\|_{  \Lambda_{h-1}^{-1}} \right].
\end{align*}
Following from \Cref{prop: TV distance less than estimation error}, we conclude that 
\begin{align*}
    \mathtt{D}_{\TV}&\left(  \Pb_{\theta^*}^{\pi}(\tau_H) , \Pb_{\hat{\theta}}^{\pi}(\tau_H)\right) \lesssim \min \left\{\sqrt{\frac{\hat{\beta}}{H\iota^2\gamma^2}} \sum_{h=1}^H \mathop{\Eb}_{\tau_{h-1}\sim\Pb_{\theta^*}^{\pi}} \left[ \     \left\|\bar{\psi}^* (\tau_{h-1})\right\|_{  \Lambda_{h-1}^{-1}}  \right] ,2 \right\}.
\end{align*}
\end{proof}
\fi

\section{Proof of \Cref{thm:robust PSR}.}\label{sec:proof of thm1}
In this section, we provide the full proof for \Cref{thm:robust PSR}. Recall that $\Ec_{o} = \Ec_{\log}\cap \Ec_{\omega}\cap\Ec_{\mathrm{H}}$, where   $\Ec_{\log}$, $\Ec_{\omega}$, and $\Ec_{\mathrm{H}}$, defined in \Cref{proposition: log likelihood of true model is large}, \Cref{proposiiton: empirical distance less than log likelihood difference}, \Cref{proposition: hellinger distance less than log likelihood distance}, respectively. 

Recall that we make the following assumption.
\begin{assumption}\label{assm: low rank}
    The nominal model $\theta^*$ has rank $r$.
\end{assumption}
Thus, we have the following lemma.
\begin{lemma}[Theorem C.1 in \cite{liu2022optimistic}]
    Any low-rank decision-making problem $\theta$ with core tests $\{\mathcal{Q}_h\}_{h=1}^H$ admits a (self-consistent) predictive state representation $\theta = \{\phi_h,\Mbf_h\}_{h=1}^H$ given core tests $\{\mathcal{Q}_h\}_{h=0}^{H-1}$, such that for any $\tau_h\in\mathcal{H}_h, \omega_h\in\Omega_h$, we have
\begin{align*}
    \left\{
    \begin{aligned}
    & \psi(\tau_h) = \Mbf_h(o_h,a_h)\cdots\Mbf_1(o_1,a_1)\psi_0,\\
    & [\psi(\tau_h)]_{\ell} = \Pc^{\theta}(\mathbf{q}_{h,\ell}^{\mathrm{o}}, \tau_h^{\mathrm{o}} | \tau_h^{\mathrm{a}}, \mathbf{q}_{h,\ell}^{\mathrm{a}}), \\
    & \mbf(\omega_h)^{\top} = \phi_H^{\top}\Mbf_H(o_H,a_H)\cdots\Mbf_{h+1}(o_{h+1},a_{h+1}),\\
    & \sum_{o_{h+1}}\phi_{h+1}^{\top}\Mbf_{h+1}(o_{h+1},a_{h+1}) = \phi_h^{\top},\\
    &\Pc^{\theta}(o_h,\ldots,o_1 | a_1,\dots, a_h) = \phi_h^{\top} \psi(\tau_h),
    \end{aligned}
    \right.
\end{align*}
where $\Mbf_h: \mathcal{O}\times\mathcal{A}\rightarrow\mathbb{R}^{d \times d }$, $\phi_h\in\mathbb{R}^{d }$, and $\psi_0\in\mathbb{R}^{d }$.
In particular, $\Theta$ is the set of all low-rank problems with the same core tests $\{\mathcal{Q}_h\}_{h=1}^H$.
\end{lemma}

First, we derive the following guarantee for the difference of robust values.

\begin{corollary}\label{lemma: sub gap PSR}
    Under the event $\Ec_{o}$, we have that 
    \[ 
    \left. 
    \begin{aligned}
    &\left|V_{B_{\Pc}^1(\theta^*), R}^{\pi} - V_{B_{\Pc}^1(\hat{\theta}),R}^{\pi} \right|\leq   V_{\hat{\theta}, \hat{b}}^{\pi}, \\
    &\left|V_{B_{\Pc}^2(\theta^*), R}^{\pi} - V_{B_{\Pc}^2(\hat{\theta}),R}^{\pi} \right|\leq   C_{\Pc}^{\KL}V_{\hat{\theta}, \hat{b}}^{\pi}, \\
    &\left|V_{B_{\Tc}^1(\theta^*), R}^{\pi} - V_{B_{\Tc}^1(\hat{\theta}),R}^{\pi} \right|\leq   2C_BV_{\hat{\theta}, \hat{b}}^{\pi}, \\
    &\left|V_{B_{\Tc}^2(\theta^*), R}^{\pi} - V_{B_{\Tc}^2(\hat{\theta}), R}^{\pi} \right|\leq 2C_BC_{\Tc}^{\KL} V_{\hat{\theta}, \hat{b}}^{\pi}.
    \end{aligned}
    \right.
    \]
\end{corollary}

\begin{proof}
    The proof is finished by directly combining \Cref{lemma:Sim P TV}, \Cref{lemma:Sim P KL}, \Cref{lemma:Sim T} and Lemma 9 and Corollary 2 in \citet{huang2023provably}.
\end{proof}

\if{0}

Recall that $\Ec_{o} = \Ec_{\log}\cap \Ec_{\omega}\cap\Ec_{\mathrm{H}}$ where the three events $\Ec_{\log}$, $\Ec_{\omega}$, and $\Ec_{\mathrm{H}}$ are defined in \Cref{proposition: log likelihood of true model is large}, \Cref{proposiiton: empirical distance less than log likelihood difference}, \Cref{proposition: hellinger distance less than log likelihood distance}, respectively. We have $\Pb(\Ec_{o})\geq 1-\delta.$

\begin{lemma}[Lemma 9 and Corollary 2 in \citet{huang2023provably}]\label{lemma:bonus}
    Under event $\Ec_o$, for any reward $R$, we have,
\begin{align*}
    \left\|  \Db_{\theta^*}^{\pi}  - \Db_{\hat{\theta}}^{\pi} \right\|_1 \leq V_{ \hat{\theta}, \hat{b} }^{\pi},
\end{align*}
where $\hat{b}^k(\tau_H) = \min\left\{  \alpha \sqrt{ \sum_{h }  \left\|\bar{\hat{\psi}}_h (\tau_h) \right\|^2_{ (\hat{U}_h )^{-1}} } , 1 \right\}$,  $\alpha = \sqrt{\frac{ 4\lambda H Q_A^2d }{\gamma^4} + \frac{ 7\beta }{ \iota^2 \gamma^2}},$ and 
\(
    \hat{U}_h = \lambda I + \sum_{\tau_h\in\Dc_h^{\mathrm{g}}} \bar{\hat{\psi}}(\tau_h)\bar{\hat{\psi}}(\tau_h)^{\top}.
\)

\end{lemma}

\begin{lemma}[Lemma 10 in \citet{huang2023provably}]\label{lemma:offline empirical bonus and true bonus}
 Under event $\Ec^o$, for any $\pi$, we have
    \begin{align*}
        \Eb_{\theta^*}^{\pi}& \left[\sqrt{ \sum_{h=0}^{H-1} \left\|\bar{\hat{\psi}}_h (\tau_h) \right\|^2_{\left(\hat{U}_h\right)^{-1}} } \right] 
        \lesssim   \frac{ \sqrt{ r\beta }}{ \iota \sqrt{\lambda}}    \sum_{h=0}^{H-1} \Eb_{\theta^*}^{\pi} \left[\left\|\bar{\psi}_h^* (\tau_h) \right\|_{\left(U_h\right)^{-1}} \right] + \frac{HQ_A}{ \sqrt{\lambda} }   \left\|  \Db_{\theta^*}^{\pi} - \Db_{\hat{\theta}}^{\pi} \right\|_1,
    \end{align*}

where $\hat{U}_h = \lambda I + \sum_{\tau_h\in\Dc_h^{\mathrm{g}}} \bar{ \psi }^*(\tau_h)\bar{ \psi }^*(\tau_h)^{\top}.
$
\end{lemma}

The following lemma provides an explicit upper bound on the total variation distance between the estimated model and the true model.
\begin{lemma}[Lemma 8 in \citet{huang2023provably}]\label{lemma: offline TV distance < true bonus}
Under event $\Ec_o$, for any policy $\pi$, we have
\begin{align}
    \left\|  \Db_{\theta^*}^{\pi}  - \Db_{\hat{\theta}}^{\pi} \right\|_1 \lesssim \sqrt{\frac{ \beta }{H\iota^2\gamma^2}} \sum_{h=1}^H  \Eb_{\theta^*}^{\pi} \left[ \left\|\bar{\psi}^* (\tau_{h})\right\|_{  \Lambda_{h}^{-1}}  \right],
\end{align}
where $\Lambda_{h} = \lambda_0 I + \frac{N}{2H} \Eb_{\theta^*}^{\rho} \left[ \bar{\psi}^*(\tau_{h}) \bar{\psi}^*(\tau_{h})^{\top} \right] $, and $\lambda_0 = \frac{\gamma^4}{4Q_A^2d}$.  
\end{lemma}
\fi

\begin{theorem}[Restatement of \Cref{thm:robust PSR}]
Suppose \Cref{assm: low rank} and \Cref{assmp:well-condition} hold. The type-I coefficient of the behavior policy $\rho$ against the robust policy $\pi^*$ is finite. Let  $Q_A = \max_h|\mathcal{Q}_h^A|$ be the maximum number of core action sequences. Let $\iota = \min_{\mathbf{q}_{h,\ell}^{\mathrm{a}}} \rho(\mathbf{q}_{h,\ell}^{\mathrm{a}}) $, $p_{\min} = \frac{\delta}{N(|\mathcal{O}||\mathcal{A}|)^{2H}}$, $\varepsilon= \frac{p_{\min}}{NH}$, $ \beta = O(\log(\mathcal{N}_{\varepsilon}(\Theta)/\delta))$, $ \lambda = H^2Q_A^2$, and $ \alpha  = O\left(\frac{  Q_A\sqrt{dH} }{\gamma^2}\sqrt{\lambda} + \frac{  \sqrt{\beta} }{ \iota \gamma} \right)$. Then, with probability at least $1-\delta$, the output $\hat{\pi}$ under different settings of \Cref{alg: robust PSR} satisfies that
    \begin{align*}
        &V_{B_{\Pc}^1(\theta^*), R}^{\pi^*} - V_{B_{\Pc}^1(\theta^*), R}^{ \hat{\pi}} \lesssim \frac{H^2 Q_A\sqrt{d\beta}}{\iota \gamma^2} \left(\sqrt{r} + \frac{Q_A\sqrt{H}}{\gamma} \right) \sqrt{ \frac{C_{\mathrm{p}}(\pi^*|\rho)}{N} }, \\
        &V_{B_{\Pc}^2(\theta^*), R}^{\pi^*} - V_{B_{\Pc}^2(\theta^*), R}^{ \hat{\pi}} \lesssim \frac{H^2 Q_A\sqrt{d\beta}}{\iota \gamma^2} \left(\sqrt{r} + \frac{Q_A\sqrt{H}}{\gamma} \right) C_{\Pc}^{\KL}\sqrt{ \frac{C_{\mathrm{p}}(\pi^*|\rho)}{N} }, \\
        &V_{B_{\Tc}^1(\theta^*), R}^{\pi^*} - V_{B_{\Tc}^1(\theta^*), R}^{ \hat{\pi}} \lesssim \frac{H^2 Q_A\sqrt{d\beta}}{\iota \gamma^2} \left(\sqrt{r} + \frac{Q_A\sqrt{H}}{\gamma} \right) C_B\sqrt{ \frac{C_{\mathrm{p}}(\pi^*|\rho)}{N} }, \\
        & V_{B_{\Tc}^2(\theta^*), R}^{\pi^*} - V_{B_{\Tc}^2(\theta^*), R}^{ \hat{\pi}} \lesssim \frac{H^2 Q_A\sqrt{d\beta}}{\iota \gamma^2} \left(\sqrt{r} + \frac{Q_A\sqrt{H}}{\gamma} \right) C_BC_{\Tc}^{\KL}\sqrt{ \frac{C_{\mathrm{p}}(\pi^*|\rho)}{N} }.
    \end{align*}
\end{theorem}

\begin{proof}

We first prove the result for $\Pc$-type uncertainty set with $\ell_1$ distance.

Recall that $\hat{\pi} = \arg\max_{\pi} \left\{  V_{B_{\Pc}^1(\hat{\theta}),R}^{\pi}  -  V_{\hat{\theta}, \hat{b}}^{\pi} \right\}$. Let $\pi^* = \arg\max_{\pi}V_{B_{\Pc}^1(\theta^*), R}^{\pi}$. Then, by utilizing \Cref{lemma: sub gap PSR}, the suboptimal gap of $\hat{\pi}$ can be upper bounded as follows.

\begin{align*}
    V_{B_{\Pc}^1(\theta^*), R}^{\pi^*} &- V_{B_{\Pc}^1(\theta^*), R}^{ \hat{\pi}} \\
    & \leq  V_{B_{\Pc}^1(\theta^*), R}^{\pi^*} - \left(V_{B_{\Pc}^1(\hat{\theta}), R}^{\pi^*} -  V_{\hat{\theta}, \hat{b}}^{\pi^*} \right) + \left(V_{B_{\Pc}^1(\hat{\theta}), R}^{\pi^*} -  V_{\hat{\theta}, \hat{b}}^{\pi^*} \right) - \left(V_{B_{\Pc}^1(\hat{\theta}), R}^{\hat{\pi}} -  V_{\hat{\theta}, \hat{b}}^{\hat{\pi}} \right) \\
    &\leq 2V_{\hat{\theta}, \hat{b}}^{\pi^*}.
\end{align*}

By Lemma 8 and Lemma 10 in \citet{huang2023provably}
we have
\begin{align*}
    V_{\hat{\theta}, \hat{b}}^{\pi^*}&\leq V_{\theta^*,\hat{b}}^{\pi^*} + \left\|  \Db_{\theta^*}^{\pi^*} - \Db_{\hat{\theta}}^{\pi^*} \right\|_1\\
    &\overset{(a)}\lesssim  \alpha\frac{ \sqrt{ r\beta }}{ \iota \sqrt{\lambda}}    \sum_{h=0}^{H-1} \Eb_{\theta^*}^{\pi} \left[\left\|\bar{\psi}_h^* (\tau_h) \right\|_{ U_h ^{-1}} \right] + \frac{\alpha HQ_A}{ \sqrt{\lambda} }   \left\|  \Db_{\theta^*}^{\pi} - \Db_{\hat{\theta}}^{\pi} \right\|_1 \\
    &\lesssim \alpha  \frac{ \sqrt{ r\beta }}{ \iota \sqrt{\lambda}}   \sum_{h=0}^H  \Eb_{\theta^*}^{\pi} \left[ \left\|\bar{\psi}^* (\tau_{h})\right\|_{ \bar{U}_{h}^{-1}}  \right]  +  \frac{ \alpha HQ_A}{ \sqrt{\lambda} }  \sqrt{\frac{ \beta }{H\iota^2\gamma^2}}  \sum_{h=0}^H  \Eb_{\theta^*}^{\pi} \left[ \left\|\bar{\psi}^* (\tau_{h})\right\|_{  \Lambda_{h}^{-1}}  \right],
\end{align*}
where $(a)$ is due to $\lambda<H^2Q_A^2$, $\alpha = \sqrt{\frac{ 4\lambda H Q_A^2d }{\gamma^4} + \frac{ 7\beta }{ \iota^2 \gamma^2}}$, and 
\begin{align*}
&\hat{U}_h = \lambda I + \sum_{\tau_h\in\Dc_h^{\mathrm{g}}} \bar{ \psi }^*(\tau_h)\bar{ \psi }^*(\tau_h)^{\top},\\
&\bar{U}_{h} = \lambda I + |\Dc_h^{\mathrm{g}}|\Eb_{\theta^*}^{\rho}\left[\bar{\psi}^*(\tau_h)\bar{\psi}^*(\tau_h)^{\top} \right], \\
&\Lambda_{h} = \lambda_0 I + \frac{N}{2H}\Eb_{\theta^*}^{\rho}\left[\bar{\psi}^*(\tau_h)\bar{\psi}^*(\tau_h)^{\top} \right],\\
&\lambda_0 = \frac{\gamma^4}{4Q_A^2d}.
\end{align*}

Recall that type-I concentrability coefficient implies that
\[   \Eb_{\theta^*}^{\pi^*}\left[\bar{\psi}^*(\tau_h)\bar{\psi}^*(\tau_h)^{\top} \right] \leq C_{\mathrm{p}}(\pi^*|\rho) \Eb_{\theta^*}^{\rho}\left[\bar{\psi}^*(\tau_h)\bar{\psi}^*(\tau_h)^{\top} \right]. \]

Thus, due to Cauchy's inequality, we have
\begin{align*}
     V_{B_{\Pc}^1(\theta^*), R}^{\pi^*} &- V_{B_{\Pc}^1(\theta^*), R}^{ \hat{\pi}}\\
     &\lesssim \alpha\frac{\sqrt{\beta}}{\iota\sqrt{\lambda}}\left(\sqrt{r} + \frac{Q_A\sqrt{H}}{\gamma} \right) \sum_h \sqrt{ C_{\mathrm{p}}(\pi^*|\rho) \Eb_{\theta^*}^{\rho}\left[ \left\|\bar{\psi}^* (\tau_{h})\right\|^2_{  \Lambda_{h}^{-1}}  \right] }\\
     &\leq \alpha H\frac{\sqrt{\beta}}{\iota\sqrt{\lambda}}\left(\sqrt{r} + \frac{Q_A\sqrt{H}}{\gamma} \right) \sqrt{ \frac{H  C_{\mathrm{p}}(\pi^*|\rho)}{N} }\\
     &\lesssim \frac{H^2 Q_A\sqrt{d\beta}}{\iota \gamma^2} \left(\sqrt{r} + \frac{Q_A\sqrt{H}}{\gamma} \right) \sqrt{ \frac{C_{\mathrm{p}}(\pi^*|\rho)}{N} }.
\end{align*}

For the other three inequalities, we follow the same argument as above and the proof is finished.

\end{proof}

\section{Proof of \Cref{thm:general robust RL}}\label{sec:proof alg 2}
We first provide the pseudo code for our proposed algorithm.

\begin{algorithm}[h]
\caption{Robust Offline RL}
\label{alg: robust}
\begin{algorithmic}[1]
\STATE {\bf Input:} offline dataset $\mathcal{D}$
 \STATE Construct confidence set $  \mathcal{C}$ according to \Cref{eqn: confidence set}.
\STATE Planning $\hat{\pi}$ such that
$\hat{\pi} = \arg\max_{ \pi \in \Pi} \min_{ \hat{\theta}  \in \mathcal{C} }   V_{B(\hat{\theta}), R}^{\pi}$.

\end{algorithmic}
\end{algorithm}

\begin{theorem}
Suppose the type-II coefficient of the behavior policy $\rho$ against the robust policy $\pi^*$ is finite. Let $\varepsilon= \frac{1}{NH}$, $ \beta = O(\log(\mathcal{N}_{\varepsilon}(\Theta)/\delta))$. Then, with probability at least $1-\delta$, the output $\hat{\pi}$ under different settings of \Cref{alg: robust} satisfies that
\begin{align*}
     &V_{B_{\Pc}^1(\theta^*), R}^{\pi^*} - V_{B_{\Pc}^1(\theta^*), R}^{\hat{\pi}}\lesssim H\sqrt{  \frac{C(\pi^*|\rho) \beta }{N}},\\
     &V_{B_{\Pc}^2(\theta^*), R}^{\pi^*} - V_{B_{\Pc}^2(\theta^*), R}^{\hat{\pi}}\lesssim HC_{\Pc}^{\KL}\sqrt{  \frac{C(\pi^*|\rho) \beta }{N}}, \\
     &V_{B_{\Pc}^1(\theta^*), R}^{\pi^*} - V_{B_{\Pc}^1(\theta^*), R}^{\hat{\pi}}\lesssim HC_B\sqrt{  \frac{C(\pi^*|\rho) \beta }{N}},\\
     &V_{B_{\Pc}^2(\theta^*), R}^{\pi^*} - V_{B_{\Pc}^2(\theta^*), R}^{\hat{\pi}}\lesssim HC_B C_{\Pc}^{\KL}\sqrt{  \frac{C(\pi^*|\rho) \beta }{N}}.
\end{align*}
    
\end{theorem}

\begin{proof}
We first prove the result for the setting when the uncertainty set is $\Pc$-type with $\ell_1$ distance.

Let $\hat{\theta} = \arg\min_{ \theta\in\mathcal{C} } V_{B_{\Pc}^1(\theta), R}^{\hat{\pi}} $. Then, we perform the decomposition as follows.
\begin{align*}
    V_{B_{\Pc}^1(\theta^*), R}^{\pi^*} &- V_{B_{\Pc}^1(\theta^*), R}^{\hat{\pi}} \\
    & = V_{B_{\Pc}^1(\theta^*), R}^{\pi^*} - V_{B_{\Pc}^1(\hat{\theta}), R}^{\pi^*} + \underbrace{V_{B_{\Pc}^1(\hat{\theta}), R}^{\pi^*}  -  V_{B_{\Pc}^1(\hat{\theta}), R}^{\hat{\pi}} }_{I_1} + \underbrace{ V_{B_{\Pc}^1(\hat{\theta}), R}^{\hat{\pi}} - V_{B_{\Pc}^1(\theta^*), R}^{\hat{\pi}} }_{I_2} \\
    &\leq V_{B_{\Pc}^1(\theta^*), R}^{\pi^*} - V_{B_{\Pc}^1(\hat{\theta}), R}^{\pi^*},
\end{align*}
where $I_1$ is due to the optimality $\hat{\pi}$, and $I_2$ follows from the property of $\hat{\theta}$.

By \Cref{lemma:Sim P TV}, we have
\begin{align*}
    V_{B_{\Pc}^1(\theta^*), R}^{\pi^*} &- V_{B_{\Pc}^1(\hat{\theta}), R}^{\pi^*} \\
    &\leq \left\|\Db_{\theta^*}^{\pi^*} - \Db_{\hat{\theta}}^{\pi^*} \right\|_1 \\
    &\overset{(a)}\leq \sum_{h=1}^H \Eb_{\theta^*}^{\pi^*} \left[ \left\|\Tc^{\theta^*}(\cdot|\tau_h) - \Tc^{\hat{\theta}}(\cdot|\tau_h) \right\|_1 \right] \\
    &\overset{(b)}\leq \sum_{h=1}^H \sqrt{ C(\pi^*|\rho)\Eb_{\theta^*}^{\rho} \left[ \left\|\Tc^{\theta^*}(\cdot|\tau_h) - \Tc^{\hat{\theta}}(\cdot|\tau_h) \right\|_1^2 \right] } \\
    &\overset{(c)}\lesssim \sum_{h=1}^H \sqrt{ C(\pi^*|\rho) \Eb_{\theta^*}^{\rho} \left[ \mathtt{D}_{\mathrm{H}}^2 \left(\Tc^{\theta^*}(\cdot|\tau_h) , \Tc^{\hat{\theta}}(\cdot|\tau_h) \right) \right] } \\
    &\overset{(d)}\lesssim H  \sqrt{  C(\pi^*|\rho)  \mathtt{D}_{\mathrm{H}}^2 \left(\Db_{\theta^*}^{\rho} , \Db_{\hat{\theta}}^{\rho}  \right)   } \\
    &\leq H\sqrt{ \frac{\beta C(\pi^*|\rho) }{N}},  
\end{align*}
where $(a)$ follows from \Cref{lemma: E TV < TV}, $(b)$ is due to the definition of type-II concentrability coefficient, and $(c),(d)$ follows from \Cref{lemma:TV and hellinger}.

The other three results follows the same argument by utilizing \Cref{lemma:Sim P KL} and \Cref{lemma:Sim T}.
\end{proof}

\section{Auxiliary Lemmas}

\begin{lemma}\label{lemma: E TV < TV}
    For two distributions $P$ and $Q$, we have
    \[ \Eb_{X\sim P_X} \left[ \left\|P_Y(\cdot|X) - Q_Y(\cdot|X) \right\|_1 \right] \leq 2\|P_{X,Y} - Q_{X,Y}\|_1.\]
    In addition, the following inequality holds.
    \[ \|P_{X,Y} - Q_{X,Y}\|_1\leq \Eb_{P}[\|P_{Y|X}(\cdot|X) - Q_{Y|X}(\cdot|X)\|_1] + \|P_X-Q_X\|_1.\]
\end{lemma}

\begin{proof}
We have
\begin{align*}
    \Eb_{X\sim P_X}&\left[ \left\|P_Y(\cdot|X) - Q_Y(\cdot|X) \right\|_1 \right] \\
    & = \sum_x P_X(x) \sum_y| P_{Y|X}(y|x) - Q_{Y|X}(y|x)  | \\
    & = \sum_x \sum_y| P_{Y|X}(y|x)P_X(x) - Q_{Y|X}(y|x) P_X(x)  | \\
    & = \sum_x \sum_y| P_{Y|X}(y|x)P_X(x) - Q_{Y|X}(y|x) Q_X(x) + Q_{Y|X}(y|x) Q_X(x) -  Q_{Y|X}(y|x) P_X(x)  | \\
    & \leq \sum_x \sum_y| P_{Y|X}(y|x)P_X(x) - Q_{Y|X}(y|x) Q_X(x) | \\
    &\quad + \sum_x \sum_y| Q_{Y|X}(y|x) Q_X(x) -  Q_{Y|X}(y|x) P_X(x)  | \\
    & = \|P_{X,Y} - Q_{X,Y}\|_1 + \sum_x|P_X(x) - Q_X(x)| \\
    & = \|P_{X,Y} - Q_{X,Y}\|_1 + \sum_x|\sum_y P_{Y|X}(y|x)P_X(x) - \sum_yQ_{Y|X}(y|x) Q_X(x)|\\
    &\leq 2\|P_{X,Y} - Q_{X,Y}\|_1.
\end{align*}

On the other hand,
\begin{align*}
    \|P_{X,Y} &- Q_{X,Y}\|_1\\
    &= \sum_x \sum_y| P_{Y|X}(y|x)P_X(x) - Q_{Y|X}(y|x) Q_X(x) | \\
    & = \sum_x \sum_y| P_{Y|X}(y|x)P_X(x) - Q_{Y|X}(y|x) P_X(x) + Q_{Y|X}(y|x) P_X(x) -  Q_{Y|X}(y|x) Q_X(x)  |  \\
    &\leq \Eb_{P}[\|P_{Y|X}(\cdot|X) - Q_{Y|X}(\cdot|X)\|_1] + \|P_X-Q_X\|_1.
\end{align*}
\end{proof}

\begin{lemma}\label{lemma:TV and hellinger}
Given two bounded measures $P$ and $Q$ defined on the set $\mathcal{X}$. Let $|P| = \sum_{x\in\mathcal{X}} P(x)$ and $|Q| = \sum_{x\in\mathcal{X}}Q(x).$ We have
\[ \|P-Q\|_1^2  \leq 4(|P|+|Q|)\mathtt{D}_{\mathtt{H}}^2(P,Q)  \]
In addition, if $P_{Y|X}, Q_{Y|X}$ are two conditional distributions over a random variable $Y$, and $P_{X,Y} = P_{Y|X}P$, $Q_{X,Y}= Q_{Y|X}Q$ are the joint distributions when $X$ follows the distributions $P$ and $Q$, respectively, we have
\[\mathop{\Eb}_{X\sim P }\left[ \mathtt{D}_{\mathtt{H}}^2(P_{Y|X},Q_{Y|X})\right] \leq 8\mathtt{D}_{\mathtt{H}}^2 (P_{X,Y},Q_{X,Y}). \]
\end{lemma}
 
\begin{proof}
We first prove the first inequality.  By the definition of total variation distance, we have
    \begin{align*}
        \|P-Q\|_1^2 & = \left(\sum_x |P(x) - Q(x)| \right)^2\\
        & = \left(\sum_x \left(\sqrt{P(x)} - \sqrt{Q(x)}\right)\left(\sqrt{P(x)} + \sqrt{Q(x)} \right) \right)^2\\
        &\overset{(a)}\leq \left(\sum_x\left(\sqrt{P(x)} - \sqrt{Q(x)}\right)^2\right) \left( 2\sum_{x}\left( P(x) + Q(x)\right) \right)\\
        &\leq 4(|P| + |Q|) \mathtt{D}_{\mathtt{H}}^2(P,Q),
    \end{align*}
where $(a)$ follows from the Cauchy's inequality and because $(a+b)^2\leq 2a^2+2b^2$.

For the second inequality, we have,
\begin{align*}
    \mathop{\Eb}_{X\sim P } &\left[ \mathtt{D}_{\mathtt{H}}^2(P_{Y|X},Q_{Y|X})\right]\\
    &= \sum_{x}P(x)\left( \sum_y \left(\sqrt{P_{Y|X}(y)} - \sqrt{Q_{Y|X}(y)} \right)^2 \right) \\
    & = \sum_{x,y} \left( \sqrt{P_{X,Y}(x,y)} - \sqrt{Q_{X,Y}(x,y)} + \sqrt{Q_{Y|X}(y) Q(x)} - \sqrt{Q_{Y|X}(y) P(x)}  \right)^2 \\
    &\leq 2\sum_{x,y} \left( \sqrt{P_{X,Y}(x,y)} - \sqrt{Q_{X,Y}(x,y)} \right)^2 + 2\sum_{x,y} Q_{Y|X}(y)\left( \sqrt{  Q(x)} - \sqrt{  P(x)}  \right)^2 \\
    & = 4\mathtt{D}_{\mathtt{H}}^2(P_{X,Y},Q_{X,Y}) + 2(|P|+|Q| - 2\sum_x\sqrt{P(x)Q(x)}) \\
    &\overset{(a)}\leq 4\mathtt{D}_{\mathtt{H}}^2(P_{X,Y},Q_{X,Y}) + 2(|P|+|Q| - 2\sum_x\sum_y\sqrt{P_{Y|X}(y)P(x)Q_{Y|X}(y)Q(x)}) \\
    & = 8 \mathtt{D}_{\mathtt{H}}^2(P_{X,Y},Q_{X,Y}),
\end{align*}
where $(a)$ follows from the Cauchy's inequality that applies on $\sum_y\sqrt{P_{Y|X}(y)Q_{Y|X}(y)}$.
\end{proof}

\end{document}